\documentclass{article}

\usepackage{graphicx}
\usepackage{arxiv}
\usepackage{natbib}
\usepackage[normalem]{ulem}
\usepackage{bm}

\usepackage[utf8]{inputenc}
\usepackage[T1]{fontenc}
\usepackage{hyperref}
\usepackage{url}
\usepackage{xcolor}
\usepackage{wrapfig}

\usepackage{amsmath}
\usepackage{amsfonts}
\usepackage{amssymb}
\usepackage{amsthm}
\usepackage[mathscr]{eucal}

\usepackage{enumitem}
\setlist[itemize]{leftmargin=1.5pc}

\theoremstyle{definition}
\newtheorem{assumption}{Assumption}
\newtheorem{remark}{Remark}
\newtheorem{definition}{Definition}

\theoremstyle{theorem}
\newtheorem{theorem}{Theorem}
\newtheorem{proposition}{Proposition}
\newtheorem{lemma}{Lemma}
\newtheorem{corollary}{Corollary}

\colorlet{sgreen}{black!65!green}

\newcommand{\sk}[1]{#1}
\newcommand{\skr}[1]{#1}

\newcommand{\brho}[1]{\bar{\rho}_{#1}}

\newcommand{\NN}[1]{\mathbf{n}_{#1}}
\newcommand{\NNI}[1]{\mathbf{n}_{#1}}
\renewcommand{\P}{\mathbb{P}}
\newcommand{\Log}{\log}

\newcommand{\Abound}{\varepsilon}
\newcommand{\Aconst}{C_{0}}

\newcommand{\thisHypName}{}
\newtheorem*{genericHyp}{\thisHypName}

\newcommand{\semiMetric}{{\rm dist}}
\newcommand{\semiDist}[2]{\semiMetric \!\left( #1 , #2 \right)}
\newcommand{\KLDiv}[2]{\mathcal{D}_{\text{kl}} \!\left(#1 | #2 \right)}

\newcommand{\meas}{\mu}

\newcommand{\F}{\mathcal{F}}

\newcommand{\W}{\mathcal{W}}
\newcommand{\X}{\mathcal{X}}
\newcommand{\Y}{\mathcal{Y}}

\newcommand{\Hyp}{\mathcal{H}}

\renewcommand{\H}{\Hyp}

\newcommand{\V}{d_\Hyp}
\newcommand{\vc}{\V}

\newcommand{\E}{\mathcal{E}}
\newcommand{\EE}{\mathbb{E}}

\newcommand{\expec}{\mathbb{E} }

\newcommand{\abs}[1]{\left|#1\right|}
\newcommand{\paren}[1]{\left(#1\right)}
\newcommand{\braces}[1]{\left\{#1\right\}}
\newcommand{\brackets}[1]{\left [ #1 \right ]}

\newcommand{\prob}{\mathbb{P}}

\newcommand{\Q}{Q}
\newcommand{\Qt}{{\cal D}}
\newcommand{\np}{n_P}

\usepackage{dsfont}
\DeclareSymbolFont{bbold}{U}{bbold}{m}{n}
\DeclareSymbolFontAlphabet{\mathbbold}{bbold}
\newcommand{\ind}[1]{{\mathbbold 1}\!\left\{#1\right\}}

\DeclareMathOperator*{\argmin}{argmin}

\newcommand{\nats}{\mathbb{N}}
\newcommand{\bE}{\mathbf{E}}
\newcommand{\compl}{\mathsf{c}}
\newcommand{\M}{\mathcal{M}}

\newcommand{\hstar}{h^{\!*}}

\newsavebox{\savepar}

\title {A No-Free-Lunch Theorem for MultiTask Learning}
\author{
 Steve Hanneke \\
   Toyota Technological Institute at Chicago\\
   \texttt{steve.hanneke@gmail.com}
   \And
 Samory Kpotufe \\
   Columbia University, Statistics\\
   \texttt{skk2175@columbia.edu}
}

\begin{document}
\maketitle

\begin{abstract}
{Multitask learning and related areas such as multi-source domain adaptation address modern settings where datasets from $N$ related distributions $\braces{P_t}$ are to be combined towards improving performance on any single such distribution $\Qt$. A perplexing fact remains in the evolving theory on the subject: 
while we would hope for performance bounds that account for the contribution from multiple tasks, 
the vast majority of analyses result in bounds that improve at best in the number $n$ of samples per task, but most often do not improve in $N$. As such, it might seem at first that the distributional settings or aggregation procedures considered in such analyses might be somehow unfavorable; however, as we show, the picture happens to be more nuanced, with interestingly hard regimes that might appear otherwise favorable. 

In particular, we consider a seemingly favorable classification scenario where all tasks $P_t$ share a common optimal classifier $\hstar$, and which can be shown to admit a broad range of regimes with improved \emph{oracle rates} in terms of $N$ and $n$. Some of our main results are as follows: 
\begin{itemize} 
\item We show that, even though such regimes admit minimax rates accounting for both $n$ and $N$, no adaptive algorithm exists; that is, without access to distributional information, \emph{no} algorithm can guarantee rates that improve with large $N$ for $n$ fixed. 
\item With a bit of additional information, namely, a ranking of tasks $\braces{P_t}$ according to their \emph{distance} to a target $\Qt$, a simple rank-based procedure can achieve near optimal aggregations of tasks' datasets, despite a search space exponential in $N$. Interestingly, the optimal aggregation might exclude certain tasks, even though they all share the same $\hstar$.
\end{itemize}}

\end{abstract}

\section{Introduction}\label{sec:introduction}
\emph{Multitask learning} and related areas such as \emph{multi-source domain adaptation} 
address a statistical setting where multiple datasets $Z_t \sim P_t, t = 1, 2, \ldots,$ are to be aggregated towards improving performance w.r.t. a single (or any of, in the case of multitask) such distributions $P_t$. This is motivated by applications in the sciences and engineering where data availability is an issue, e.g., medical analytics typically require aggregating data from loosely related subpopulations, while  identifying traffic patterns in a given city might benefit from pulling data from somewhat similar other cities.  

While these problems have received much recent theoretical attention, especially in classification, a perplexing reality emerges: the bulk of results appear to show little improvement from such aggregation over using a single data source. Namely, given $N$ datasets, each of size $n$, one would hope for convergence rates in terms of the aggregate data size $N\cdot n$, somehow adjusted w.r.t. \emph{discrepancies} between distributions $P_t$'s, but which clearly improve on rates in terms of just $n$ as would be obtained with a single dataset. 
However, such clear improvements on rates appear elusive, as typical bounds on excess risk w.r.t. a target $\Qt$ (i.e., one of the $P_t$'s) are of the form (see e.g., \cite{crammer2008learning, ben2008notion, ben2010theory}) 
\begin{align}
\E_{\Qt}(\hat h)\lesssim \paren{n N}^{-\alpha} + n^{-\alpha} + \text{disc}\!\paren{\braces{P_t}; \Qt}, \ \text{for some } \alpha \in [{1}/{2}, 1], \label{eq:slowrates}
\end{align}
where in some results, one of the last two terms is dropped. 
In other words, typical upper-bounds are either dominated by the rate $n^{-\alpha}$, or altogether might not go to 0 with sample size due to the \emph{discrepancy} terms $\text{disc}\!\paren{\braces{P_t}; \Qt}>0$, even while the excess risk  of a naive classifier $\tilde h$ trained on the target dataset would be $\E_\Qt(\tilde h) \propto n^{-\alpha} \to 0$. As such, it might seem at first that there is a gap in either algorithmic approaches, formalism and assumptions, or statistical analyses of the problem. However, as we argue here, no algorithm can guarantee a rate improving in aggregate sample size for $n$ fixed, even under seemingly generous assumptions on how sources $P_t$'s relate to a given target $\Qt$.  

In particular, we consider a seemingly favorable classification setting, where all data distributions $P_t$'s induce the same optimal classifier $\hstar$ over a hypothesis class $\H$. This situation is of independent interest, e.g., appearing recently under the name of \emph{invariant risk minimization} (see {\cite{arjovsky2019invariant}} where it is motivated through invariants under causality), but is motivated here by our aim to elucidate basic limits of the multitask problem. As a starting point to understanding the best achievable rates, we first establish minimax upper and lower bounds, up to log terms, for the setting. These oracle rates, as might be expected in such benign settings, do indeed 
improve with both $N$ and $n$, as allowed by the level of \emph{discrepancy} between distributions, appropriately formalized (Theorem \ref{thm:minimax}). We then turn to characterizing the extent to which such favorable rates might be achieved by \emph{reasonable} procedures, i.e., adaptive procedures with little to no prior distributional information. Many interesting messages arise, some of which we highlight below: 

\begin{itemize} 
\item \emph{No adaptive procedure exists} outside a threshold $\beta < 1$, where $\beta\in [0, 1]$ (a so-called Bernstein class condition) parametrizes the level of \emph{noise}\footnote{The term \emph{noise} is used here to indicate nondeterminism in the label $Y$ of a sample point $X$, and so is rather non-adversarial.} in label distribution. Namely, while oracle rates might decrease fast in $N$ and $n$, no procedure based on the aggregate samples alone can guarantee a rate better than $n^{-1/(2-\beta)}$ without further favorable restrictions on distributions (Theorem \ref{thm:nonadaptivity}). 

\item At low noise $\beta =1$ (e.g., so called Massart's noise) even the naive, yet common approach of \emph{pooling} all datasets, i.e., treating them as if identically distributed, is nearly minimax optimal, achieving rates improving in both $N$ and $n$. This of course would not hold if optimal classifiers $\hstar$'s differ considerably across $P_t$'s (Theorem \ref{thm:beta-1-union}). 

\item At any noise level, a ranking of sources $P_t$'s according to discrepancy to a target $\Qt$ is sufficient information for (partial) adaptivity. While a precise ranking is probably unlikely in practice, an approximate ranking might be available, as domain knowledge on how sources rank in fidelity w.r.t. to an ideal target data: e.g., in settings such as 
learning with slowly drifting distributions. 
Here we show that a simple rank-based procedure, using such ranking, efficiently achieves a near optimal aggregation rate, despite the exponential search space of over $2^N$ possible aggregations (Theorem \ref{thm:semi-adaptive}). 
\end{itemize} 

Interestingly, even assuming all $\hstar$'s are the same, the optimal aggregation of datasets can change with the choice of target $\Qt$ in $\braces{P_t}$ (see Theorem \ref{theo:asymmetry}), due to the inherent asymmetry in the information tasks might have on each other, e.g., some $P_t$ might yield data in $P_s$-dense regions but not the other way around. Hence, to capture a proper picture of the problem, we cannot employ a symmetric notion of discrepancy as is common in the literature on the subject. Instead, we proceed with a notion of \emph{transfer exponent}, which we recently introduced in \cite{hanneke2019value} for the setting of domain adaptation with $N=1$ source distribution, and which we show here to also successfully capture performance limits in the present setting with $N\gg 1$ sources (see definition in Section \ref{sec:transfer}). 

We note that in the case $N=1$, there is \emph{always} a minimax adaptive procedure (as shown in \cite{hanneke2019value}), while this is not the case here for $N\gg 1$. In hindsight, the reason is simple: considering $N=O(1)$, i.e., as a constant, there is no consequential difference between a rate in terms of $N\cdot n$ and one in terms of $n$. In other words, the case $N=O(1)$ does not yield adequate insight into multitask with $N\gg 1$, and therefore does not properly inform practice as to the best approaches to aggregating and benefiting from multiple datasets.

\subsection*{Background and Related Work}
 
\begin{wrapfigure}{R}{0.25\textwidth}
\vspace{-16pt}
\centering
\includegraphics[width=0.2\textwidth]{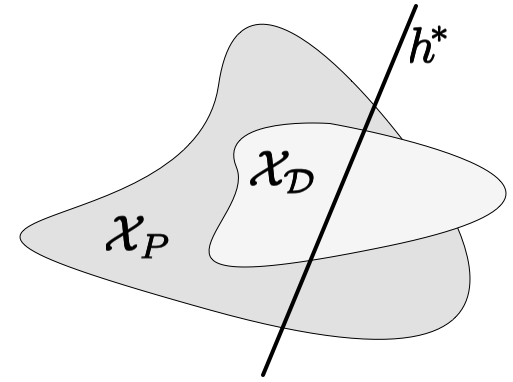}
\vspace{-12pt}
\end{wrapfigure} The bulk of theoretical work on multitask and related areas build on early work on domain adaptation (i.e., $N=1$)
such as \cite{ben2007analysis, cortes2008sample, ben2010theory}, which introduce notions of discrepancy such as the $d_{\cal A}$-divergence, $\cal Y$-discrepancy, that specialize the \emph{total-variation} metric to the setting of domain adaptation. These notions often result in bounds of the form \eqref{eq:slowrates}, starting with \cite{crammer2008learning, ben2010theory}. Such bounds can in fact be shown to be tight w.r.t. the discrepancy term $\text{disc}\!\paren{\braces{P_t}; \Qt}$, for given distributional settings and sample sizes, owing for instance to early work on the limits of learning under distribution drift (see e.g. \cite{bartlett1992learning}). However, the rates of \eqref{eq:slowrates} appear pessimistic when we consider settings of interest here where $\hstar$'s remain the same (or nearly so) across tasks, as they suggest no general improvement on risk with larger sample size. Consider for instance simple situations as depicted on the right, where a source $P$ and target $\Qt$ (with respective supports $\X_P, \X_\Qt$) might differ considerably in the mass they assign to regions of data space, thus inducing large discrepancies, but where both assign sufficient mass to decision boundaries to help identify $\hstar$ with enough samples from \emph{either} distribution. Proper parametrization resolves this issue, as we show through new multitask rates $\E_\Qt(\hat h)\to 0$ in natural situations with $\text{disc}\!\paren{\braces{P_t}; \Qt}>0$, and even with no target sample.  

We remark that other notions of discrepancy, e.g., \emph{Maximum Mean Discrepancy} \cite{gretton2009covariate}, \emph{Wasserstein distance} \cite{redko2017theoretical, shen2018wasserstein} are employed in domain adaptation; however they appear relatively less often in the theoretical literature on multitask and related areas. For multitask, 
the work of \cite{ben2008notion} proposes a more-structured 
notion of task relatedness, whereby a source $P_t$ is induced from a target $\Qt$ through a transformation of  
the domain $\X_\Qt$; that work also incurs an $n^{-1/2}$ term 
in the risk bound, but no discrepancy term.
The work of \cite{mansour2009multiple}
considers R\'enyi divergences in the context of optimal aggregation in multitask under population risk, but does not study sample complexity.

The use of a non-metric discrepancy such as R\'enyi divergence brings back an important point: two distributions might have asymmetric information on each other w.r.t. domain adaptation. Such insight was raised recently, independently in \cite{kpotufe2018marginal, hanneke2019value, achille2019information}, with various natural examples therein (see also Section \ref{sec:transfer}). In particular, it motivates a more unified view of multitask and \emph{multisource domain} adaptation, which are often treated separately. Namely, if the goal in multitask is to perform as well as possible on each task $P_t$ in our set, then as we show, such asymmetry in information between tasks calls for different aggregation of datasets for each target $P_t$: in other words, treating multitask as separate multisource problems, even if the optimal $\hstar$ is the same across tasks. 

In contrast, a frequent aim in multitask, dating back to \cite{caruana1997multitask}, has been to ideally arrive at a single aggregation of task datasets that simultaneously benefits all tasks $P_t$'s. Following this spirit, many theoretical works on the subject are concerned with bounds on \emph{average risk} across tasks \cite{Baxter-Bayesian,ando2005framework,maurer2013sparse,pontil2013excess,pentina2014pac,yang2013theory, maurer2016benefit}, rather than bounding the supremum risk accross tasks -- i.e., treating multitask as separate multisources, as of interest here. Some of these average bounds, e.g., \cite{maurer2013sparse, maurer2016benefit,pontil2013excess}, remove the dependence on discrepancy inherent in bounds of the form \eqref{eq:slowrates}, but maintain a term of the form $n^{-\alpha}$; in other words, any bound on supremum risk derived from such results would be in terms of a single dataset size $n$. 
The work of \cite{blum2017collaborative}
directly addresses the problem 
of bounding the supremum risk, 
but also incurs a term 
of the form $n^{-\alpha}$ 
in the risk bound.

In the context of multisource domain adaptation, it has been recognized in practice that some datasets that are too \emph{far} from the ideal target might hurt target performance and should be downweighted accordingly, a situation that has been termed \emph{negative transfer}. These situations further motivate the need for \emph{adaptive} procedures that can automatically identify good datasets. As far as theoretical insights, it is clear that negative transfer might happen for instance, under ERM, if optimal classifiers $\hstar$'s are considerably different across tasks.
Interestingly, even when $\hstar$'s are allowed arbitrarily close (but not equal), 
\cite{david2010impossibility} shows that, for $N=1$, the source dataset can be useless without labeled target data. We later derived minimax lower-bounds for the case $N=1$ with or without labeled target data in \cite{hanneke2019value} for a range of situations including those considered in \cite{david2010impossibility}. Such results however do not quite confirm negative transfer, as they allow the possibility that useless datasets might remain safe to include. 
For the multisource case $N\gg 1$, to the best of our knowledge, situations of negative transfer have only been described in adversarial settings with corrupted labels. For instance, the recent papers of \cite{qiao2018outliers, mahloujifar2019universal, konstantinov2020sample} show limits of multitask under various adversarial corruption of labels in datasets, while \cite{scott2019learning} derives a positive result, i.e., rates (for Lipschitz loss) decreasing in both $N$ and $n$, up to excluded or downweighted datasets. The procedure of \cite{scott2019learning} is however nonadaptive as it requires known noise proportions. 

\cite{konstantinov2020sample} is of particular interest as they are concerned with \emph{adaptivity} under label corruption. They show that, even if at most $N/2$ of the datasets are corrupted, no procedure can get a rate better than $1/n$. In contrast, in the stochastic setting considered here, there is always an adaptive procedure with significantly faster rates than $1/n$ if at most $N/2$ (in fact any fixed fraction) of distributions $P_t$'s are \emph{far} from the target $\Qt$ (Theorem \ref{thm:pooling-median} of Appendix \ref{app:pooling-median}). This dichotomy is due to the strength of adversarial corruptions they consider which in effect can flip optimal $\hstar$'s on corrupted sources. What we show here is that, even when $\hstar$'s are fixed across tasks, and datasets are sampled i.i.d. from each $P_t$, i.e., non-adversarially, no algorithm can achieve a rate better than $1/\sqrt{n}$ while a non-adaptive oracle procedure can (see Theorem \ref{thm:nonadaptivity}). In other words, some datasets are indeed unsafe for \emph{any} multisource procedure even in nonadversarial settings, as they can force suboptimal choices w.r.t. a target, absent additional restrictions on the problem setup.

As discussed earlier, such favorable restrictions concern for instance situations where information is available on how sources rank in distance to a target. In particular, the case $n=1$ intersects with the so-called \emph{distribution drift} setting where each distribution $P_t, t\in [N]$ has bounded discrepancy $\sup_{t\geq 1}\text{dist}(P_t, P_{t+1})$ w.r.t. the next distribution $P_{t+1}$ as \emph{time} $t$ varies. 
While the notion of discrepancy is typically taken to be total-variation \cite{ben1989parametrization, bartlett1992learning, barve1997complexity} 
or related notions \cite{mohri2012new, zimintasks, hanneke19a}, both our upper and lower-bounds, specialized to the case $n=1$, provide a new perspective for distribution drift under our distinct parametrization of how $P_t$'s relate to each other. For instance, our results on multisource under ranking (Theorem \ref{thm:semi-adaptive}) imply new rates for distribution drift, when all $\hstar$'s are the same across time, with excess error at time $N$ going to $0$ with $N$, even in situations where
$\sup_{t\geq 1}\text{dist}(P_t, P_{t+1}) >0$ in total variation. Such consistency in $N$ is unavailable in prior work on distribution drift.

Finally we note that there has been much recent theoretical efforts towards other formalisms of relations between distributions in a multitask or multisource scenario, with particular emphasis on distributions sharing common latent substructures, both as applied to classification \cite{muandet2013domain, mcnamara2017risk, arora2019theoretical}, or to regression settings \cite{jalali2010dirty,lounici2011oracle, negahban5773043, du2020few,tripuraneni2020theory}. The present work does not address such settings.

\paragraph{Paper Outline} We start with setup and definitions in Sections \ref{sec:setupclass} and \ref{sec:multitask}. This is followed by a technical overview of results, along with discussions of the analysis and novel proof techniques, in Section \ref{sec:overview}. Minimax lower and upper bounds are derived in Sections \ref{sec:lowerbound} and \ref{sec:upperbound}. Constrained regimes allowing partial adaptivity are discussed in Section \ref{sec:semiadaptive}. This is followed by impossibility theorems for adaptivity in Section \ref{sec:nonadaptive}.

\section{Basic Classification Concepts}\label{sec:setupclass}
We consider a classification setting $X\mapsto Y$, where $X, Y$ are drawn from  
some space $\X \times \Y$, {$\Y \doteq \braces{-1, 1}$}. We focus on \emph{proper learning} where, given data, a learner is to return a classifier $h: \X \mapsto \Y$ from some fixed class $\H \doteq \braces{h}$.  

\begin{assumption}[Bounded VC]
Throughout we will let $\H \in 2^\X$ denote a hypothesis class of finite VC dimension $\vc$, which we consider fixed in all subsequent discussions. 
To focus on nontrivial 
cases, we assume 
$|\H| \geq 3$ throughout.
\end{assumption}

We note that our algorithmic techniques and analysis extend to more general $\H$ through  
Rademacher complexity or empirical covering numbers. We focus on VC classes for simplicity, and to allow 
simple expressions of minimax rates. 

The performance of any classifier $h$ will be captured through the 0-1 risk and excess risk as defined below. 

\begin{definition}  Let $R_P(h) \doteq P_{X, Y}\paren{h(X)\neq Y}$
denote the {\bf risk} of any $h: \X \mapsto \Y$ under a distribution $P = P_{X, Y}$.
The {\bf excess risk} of $h$ over any $h'\in \H$ is then defined as 
$\E_P(h; h') \doteq R_P(h) - R_P(h')$, while for the excess risk over the best in class we simply write 
$\E_P(h) \doteq R_P(h) - \inf_{h'\in \H}R_P(h')$. We let $\hstar_P$ denote any element of $\argmin_{h\in \H}R_P(h)$ (which we will assume exists); if $P$ is clear from context we might just write $\hstar$ for a minimizer (a.k.a. \emph{best in class}).
Also define the pseudo-distance 
$P(h \neq h') \doteq P_{X}(h(X) \neq h'(X))$.

Given a finite dataset $S$ of $(x, y)$ pairs in $\X \times \Y$, we let $\hat{R}_{S}(h) \doteq \frac{1}{|S|} \sum_{(x,y) \in S} \ind{h(x) \neq y}$
denote 
the {\bf empirical risk} of $h$ under $S$; 
if $S = \emptyset$, define $\hat{R}_{S}(h) \doteq 0$. The {\bf excess empirical risk} over any $h'$ is $\hat \E_S(h; h') \doteq \hat R_S(h) - \hat R_S(h')$.
Also define the empirical pseudo-distance $\hat \prob_{S}( h \neq h' ) \doteq \frac{1}{|S|} \sum_{(x,y) \in S} \ind{h(x) \neq h'(x)}$. 
\end{definition}

The following condition is a classical way to capture a continuum from easy to hard classification.
In particular, in vanilla classification, the best rate excess risk achievable by a classifier trained on data of size $m$, can be shown to be of order $m^{-1/(2-\beta)}$, i.e, interpolates between $1/n$ and $1/\sqrt{n}$, as controlled by $\beta\in [0, 1]$ defined below. 

\begin{definition} 
\label{def:beta}
Let $\beta \in [0, 1], C_\beta \geq 2$. A distribution $P$ is said to satisfy a {\bf Bernstein class condition} with parameters $\paren{C_\beta, \beta}$ if the following holds: 
\begin{align}
\forall h\in \H, \quad  P(h \neq \hstar_P) \leq C_\beta \cdot \E_P^\beta (h). \label{eq:bernstein}
\end{align}
\end{definition}

Notice that the above always holds with at least $\beta = 0$.

The condition can be viewed as quantifying the \emph{amount of noise} in $Y$ since always have $P_X(h \neq \hstar_P) \geq \E_P(h)$ with equality when $Y = \hstar(X)$. In particular, it captures the so-called \emph{Tsybakov noise margin condition} when the Bayes classifier is in $\H$, that is, let $\eta(x) \doteq \expec_P [Y \mid X = x]$, then the margin condition  
$$P_X\paren{x: \abs{\eta(x) - 1/2} \leq \tau } \leq C\tau^\kappa, \forall \tau > 0, \text{ and some } \kappa \geq  0,$$ implies that $P$ satisfies 
\eqref{eq:bernstein} with $\beta = \kappa/(1+ \kappa)$ for some $C_\beta$ {\cite{tsybakov2004optimal}}.

The importance of such margin considerations become evident as we consider which multitask learner is able to automatically adapt optimally to unknown relations between distributions; interestingly, hardness of adaptation has to do not only with relations (or discrepancies) between distributions, but also with $\beta$. 

\section{Multitask Setting}\label{sec:multitask}
We consider a setting where multiple datasets are drawn independently from (related) distributions $P_t, t \in [N+1]$, with the aim to return hypotheses $h$'s from $\H$ with low excess error $\E_{P_t}(h)$ w.r.t. any of these distributions. W.l.o.g. we fix attention to a single \emph{target} distribution $\Qt \doteq P_{N+1}$, and therefore reduce the setting to that of \emph{multisource}\footnote{The term \emph{multisource} is often used for situations where the learner has no access to target data, which is not required to be the case here, although is handled simply by setting the target sample size to $0$ in our bounds.}.

\begin{definition} A {\bf multisource learner} $\hat h$ is any function $\prod_{t \in [N+1]}(\X \times \Y)^{{n}_{t}} \mapsto \H$ (for some sample sizes ${n}_{t} \in \mathbb{N}$), i.e., given a multisample $Z = \braces{Z_t}_{t \in [N+1]}, \abs{Z_t} = {n}_{t}$, returns a hypothesis in $\H$. 
In an abuse of notation, we will often conflate the learner $\hat h$ with the hypothesis it returns.
\end{definition}

A multitask setting can then be viewed as one where $N+1$ multisource learners $\hat h_t$ (targeting each $P_t$) are trained on the same multisample $Z$. Much of the rest of the paper will therefore discuss learners for a given target $\Qt \doteq P_{N+1}$.

\subsection{Relating Sources to Target}\label{sec:transfer}
Clearly, how well one can do in multitask depend on how distributions relate to each other.
The following parametrization will serve to capture the relation between sources and target distributions. 

\begin{definition} 
\label{def:rho}
Let $\rho>0$ (up to $\rho = \infty$), and $C_\rho \geq 2$. We say that a distribution $P$ has {\bf transfer exponent} $\paren{C_\rho, \rho}$ w.r.t.\ a distribution $\Qt$ (under $\H$), if 
$$ \forall h \in \H, \quad \E_\Qt(h) \leq C_\rho \cdot \E_P^{1/\rho}(h).$$
\end{definition}

Notice that the above always holds with at least $\rho = \infty$. 

We have shown in earlier work \cite{hanneke2019value} that the \emph{transfer exponent} manages to tightly capture the minimax rates of transfer in various situations with a single source $P$ and target $\Qt$ ($N=1$), including ones where the best hypotheseses $\hstar_P, \hstar_\Qt$ are different for source and target; in the case of main interest here where $P$ and $\Qt$ share a same best in class $\hstar \doteq \hstar_P \doteq \hstar_\Qt$, for respective data sizes $n_P, n_\Qt$, the best possible excess risk $\E_\Qt$ is of order $(n_P^{1/\rho} + n_{\Qt})^{-1/(2-\beta)}$ which is tight for any values of $\rho$ and $\beta$ (a Bernstein class parameter on $P$ and $\Qt$). In other words, $\rho$ captures an effective data size $n_P^{1/\rho}$ contributed by the source to the target: this decreases as $\rho \to \infty$, delineating a continuum from easy to hard transfer. Interestingly, $\rho<1$ reveals the fact that source data could be more useful than target data, for instance if the classification problem is easier under the source (e.g., $P$ has more mass at the decision boundary). 

Altogether, the transfer exponent $\rho$ appears as a more optimistic measure of discrepancy between source and target, as it reveals the possibility of transfer -- even at fast rates -- in many situations where traditional measures are pessimistically large. To illustrate this, we recall some of the examples from \cite{hanneke2019value}. In all examples below, we assume for simplicity that 
$Y = \hstar(X)$ for some $\hstar \doteq \hstar_P \doteq \hstar_\Qt$. 

{\bf Example 1. (Discrepancies $d_{\cal A}, d_{\cal Y}$ can be too large)} 
Let $\Hyp$ consist of one-sided thresholds on the line, and let 
$P_X \doteq \mathcal{U}[0, 2]$ and $Q_X \doteq \mathcal{U}[0, 1]$.
\begin{wrapfigure}{R}{0.25\textwidth}
\vspace{-18pt}
\centering
\includegraphics[width=0.23\textwidth]{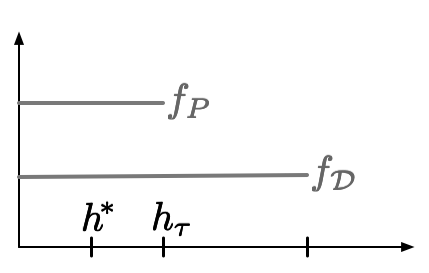}
\vspace{-18pt}
\end{wrapfigure} 
Let $\hstar$ be thresholded at $1/2$. We then see that for all $h_\tau$ thresholded at $\tau \in [0, 1]$, 
$2P_X(h_\tau \neq \hstar) = \Qt_X(h_\tau \neq \hstar)$, where for $\tau >1$, 
$P_X(h_\tau \neq \hstar) = \frac{1}{2}(\tau-1/2) \geq \frac{1}{2}\Qt_X(h_\tau \neq \hstar) =  \frac{1}{4}$. 
Thus, the transfer exponent $\rho = 1$ with {$C_\rho = 2$}, so we have fast 
transfer at the same rate $1/n_P$ as if we were sampling from $\Qt$.

On the other hand, recall that the $d_{\cal A}$-divergence takes the form\footnote{Note that these divergences are often defined w.r.t. every pair $h, h'\in \H$ rather than w.r.t.\ $\hstar$ which makes them smaller.} 
$d_{\cal A}(P, \Qt) \doteq \sup_{h \in \Hyp} \abs{P_X(h\neq \hstar) - \Qt_X(h\neq \hstar)}$, while the $\cal Y$-discrepancy takes the form $d_{\cal Y}(P, \Qt) \doteq \sup_{h\in \Hyp} \abs{\E_P(h) - \E_\Qt(h)}$. The two coincide when $Y = \hstar(X)$.  

Now, take $h_\tau$ as the threshold at $\tau = 1/2$, and $d_{\cal A} = d_{\cal Y} =\frac{1}{4}$ which would wrongly imply that transfer is not feasible at a rate faster than $\frac{1}{4}$; we can in fact make this situation worse, i.e., let $d_{\cal A} = d_{\cal Y} \to \frac{1}{2}$ by letting $\hstar$ correspond to a threshold close to $0$. A first issue is that these divergences get large in large disagreement regions; this is somewhat mitigated by \emph{localization} -- i.e., defining these discrepancies w.r.t. $h$'s in a vicinity of $\hstar$, but does not quite resolve the issue, as discussed in earlier work \cite{hanneke2019value}.

{\bf Example 2. (Minimum $\rho$, and the inherent {\bf asymmetry} of transfer)}
\begin{wrapfigure}{R}{0.25\textwidth}
\vspace{-25pt}
\centering
\includegraphics[width=0.23\textwidth]{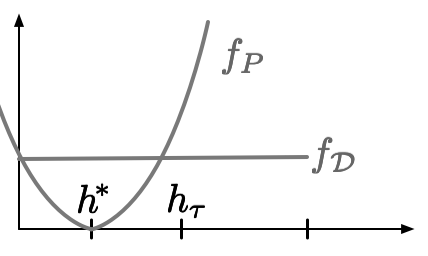}
\vspace{-18pt}
\end{wrapfigure} 
Suppose $\Hyp$ is the class of one-sided thresholds on the line, $\hstar$ is a threshold at $0$. 
The marginal $\Qt_X$ has uniform density $f_\Qt$ (on an interval containing $0$), 
while, for some $\rho\geq 1$, $P_X$ has 
density $f_P(\tau) \propto \tau^{\rho -1}$ on $\tau>0$ (and uniform on the rest of the support of $\Qt$, not shown).  
Consider any $h_\tau$ at threshold $\tau>0$, we have $P_X(h_\tau \neq \hstar) = \int_0^\tau f_P \propto \tau^\rho$, 
while $\Qt_X(h_\tau \neq \hstar) \propto \tau$. Notice that for any fixed $\epsilon>0$, 
$ \lim\limits_{\tau>0, \,  \tau \to 0} \frac{\Qt_X(h_\tau \neq \hstar)^{\rho -\epsilon}}{P_X(h_\tau \neq \hstar)} 
= \lim\limits_{\tau>0, \, \tau \to 0} C\frac{\tau^{\rho -\epsilon}}{\tau^\rho} = \infty. $

{\vskip -1mm}We therefore see that $\rho$ is the smallest possible transfer-exponent. Interestingly, now consider transferring instead 
from $\Qt$ to $P$: we would have $\rho(\Qt\to P) =1 \leq \rho \doteq \rho(P\to \Qt)$; in other words, there are natural situations where it is easier to transfer from $\Qt$ to $P$ than from $P$ to $\Qt$, as in the case here where $P$ gives relatively little mass to the decision boundary. 
This is \emph{not captured by symmetric notions of distance}, e.g., metrics or semi-metrics such as $d_{\cal A}$, $d_{\cal Y}$, MMD, TV, or Wasserstein.

Finally note that the above examples can be extended to more general 
hypothesis classes $\H$ as the examples merely play on how fast $f_P$ decreases w.r.t. $f_\Qt$ in regions of space.

\subsection{Multisource Class}
We are now ready to formalize the main class of distributional settings considered in this work.

\begin{definition}[Multisource class] We consider classes $\M = \M\paren{C_\rho, \braces{{\rho_t}}_{t\in [N]}, \braces{{n}_{t}}_{t \in [N+1]}, C_\beta, \beta}$ of product distributions of the form $\sk{\Pi = \prod_{t\in [N]} P_t^{{n}_{t}} \times \Qt^{n_{\Qt}}, \ n_t \geq 1, \ n_\Qt\doteq n_{N+1} \geq 0,}$ satisfying: 
\begin{enumerate}
    \item[(A1).] There exists $\hstar \in \H, \quad \forall t \in [N+1], \, \hstar\in \argmin_{h\in \H} \E_{P_t}(h)$, 
    \item[(A2).] Sources $P_t$'s have transfer exponent $\paren{C_{\rho}, \rho_t}$ w.r.t.\ the target $\Qt$,
    \item[(A3).] All sources $P_t$ and target $\Qt$ admit a Bernstein class condition with parameter $(C_\beta, \beta)$. 
\end{enumerate}
\end{definition}

For notational simplicity, we will often let $P_{N+1} \doteq \Qt$.
Also, although it is not a parameter of the class, we will also refer to 
$\rho_{N+1} = 1$, as $\Qt$ always has transfer exponent 
{$\paren{C_{\rho},1}$} w.r.t.\ itself.

\begin{remark}[$\hstar$ is almost unique]
Note that, for $\beta >0$, the Bernstein class condition implies that $\hstar$ above satisfies 
$P_t(\hstar\neq \hstar_t) = 0$ for any other $\hstar_t \in \argmin_{h\in \H} \E_{P_t}(h)$. Furthermore for any $t \in  [N]$ such that $\rho_t <\infty$, we also have $\E_\Qt(\hstar_t) = 0$, implying $\Qt(\hstar\neq \hstar_t) = 0$ for $\beta >0$. 
\end{remark}

\subsection{Additional Notation} 

{\bf Implicit target $\Qt = \Qt(\Pi)$.} {Every time we write $\Pi$, 
we will implicitly mean 
$\Pi = \prod_{t \in [N]} P_{t}^{n_{t}} \times \Qt^{n_{N+1}}$, 
so that in any context where 
a $\Pi$ is introduced, 
we may write, for instance, 
$\E_{\Qt}(h)$, which then refers 
to the $\Qt$ distribution in $\Pi$.}

{\bf Indices and Order Statistics.} For any $Z = \braces{Z_t}_{t \in [N+1]} \sim \Pi$, and indices $I \subset [N+1]$, we let $Z^I\doteq \bigcup_{s \in I} Z_{s}$. 

We will often be interested in order statistics $\rho_{(1)} \leq \rho_{(2)} \ldots \leq \rho_{(t)}$ of ordered $\rho_t, t\in [N+1]$ values, in which case $P_{(t)}, Z_{(t)}, n_{(t)}$ will denote distribution, sample and sample size at index ${(t)}$. We will then let $Z^{(t)}\doteq Z^{\braces{(1), \ldots, (t)}}$. 

{\bf Average transfer exponent.} For any $t\in [N+1]$, 
define $\brho{t} \doteq \sum_{s\in [t]} \alpha_{(s)}\cdot \rho_{(s)}$, where $\alpha_{(s)} \doteq \frac{n_{(s)}}{\sum_{r\in [t]} n_{(r)}}$.

{\bf Aggregate ERM.} For any $Z^{I}, I \subset [N+1]$, we let $\hat h_{Z^{I}} \doteq \argmin_{h \in \H} \hat R_{Z^{I}}(h)$, and correspondingly we also define $\hat h_{Z^{(t)}}$, $t \in [N+1]$ as the ERM over $Z^{(t)}$. When $t = N+1$ we simply write $\hat h_Z$ for $\hat h_{Z^{(N+1)}}$. 

{\bf Min and Max.} We often use the short notations $a\land b \doteq \min\braces{a, b}$, $a \lor b \doteq \max\braces {a, b}$. 

{\bf Positive Logarithm.} For any $x \geq 0$, define $\Log(x) \doteq \max\{ \ln(x), 1 \}$.

\noindent{\bf $\boldsymbol{1/0}$ Convention.} We adopt the convention 
that $1/0 = \infty$.

{\bf Asymptotic Order.} 
We often write $a \lesssim b$ or $a \asymp b$ in the statements of key results, to indicate inequality, respectively, equality, up to 
constants and logarithmic factors. The precise 
constants and logarithmic factors are always presented in supporting results. 

\section{Results Overview}\label{sec:overview}
We start by investigating what the best possible transfer rates are for multisource classes $\M$, and then investigate the extent to which these rates are attainable by adaptive procedures, i.e., procedures with little access to class information such as transfer exponents $\rho_t$ from sources to target. 

From this point on, we let $\M = \M\paren{C_\rho, \braces{{\rho_t}}_{t\in [N]}, \braces{{n}_{t}}_{t \in [N+1]}, C_\beta, \beta}$ denote any multisource class, with \emph{any admissible value} of relevant parameters, unless these parameters are specifically constrained a result's statement. 

\subsection{Minimax Rates}\label{sec:minimax}

\begin{theorem}[Minimax Rates]
\label{thm:minimax}
Let $\M$ denote any multisource class where $\rho_t \geq 1, \forall t\in [N]$. 
Let $\hat h$ denote any multisource learner with knowledge of $\M$. We have:
\begin{align*}
\inf\limits_{\hat{h}} \sup\limits_{\Pi \in \M} \EE_{\Pi} \brackets{\E_{\Qt}\!\left( \hat{h} \right)} \asymp \min_{t\in [N+1]} \paren{\sum_{s = 1}^t n_{(s)}}^{-{1}/ (2-\beta){\brho{t}}}.
\end{align*}
\end{theorem}

\begin{remark}
We remark that the proof of the above result (see Theorems {\ref{thm:lower} and \ref{thm:upper}} of Sections \ref{sec:lower} and \ref{sec:upper} for matching lower and upper bounds) imply that, in fact, we could replace $\brho{t}$ 
with simply $\rho_{(t)}$
and the result would still be true. In other words, although 
intuitively $\brho{t}$ might be much smaller than $\rho_{(t)}$ for any fixed $t$, 
the minimum values over $t \in [N+1]$ can only differ up to logarithmic terms. 
\end{remark}

{We also note that the constraint that $\rho_t \geq 1$ is only needed for the lower bound 
(Theorem~\ref{thm:lower}), whereas all of our upper bounds (Theorems~\ref{thm:upper}, \ref{thm:beta-1-union}, and \ref{thm:semi-adaptive}) hold 
for any values $\rho_t > 0$.  Moreover, there exist classes $\H$ where 
the lower bound also holds for all $\rho_t > 0$, so that the form of the bound 
is generally not improvable. 
The case $\rho_t \in (0,1)$ represents a kind of \emph{super transfer}, 
where the source samples are actually \emph{more} informative than target samples.}

It follows from Theorem \ref{thm:minimax} that, despite there being $2^{N+1}$ possible ways of aggregating datasets (or more if we consider general weightings of datasets), it is sufficient to search over $N+1$ possible aggregations -- defined by the ranking $
\rho_{(1)} \leq\rho_{(2)} \ldots \leq \rho_{(N+1)}$ -- to nearly achieve the minimax rate.

The lower-bound {(Theorem \ref{thm:lower})} relies on constructing a subset (of $\M$) of product distributions $\Pi_{h}, h \in \H$, which are mutually \emph{close} in KL-divergence, but far under the pseudo-metric $\Qt_X(h\neq h')$. For illustration, considering the case $\beta =1$, for any $h, h'$ that are sufficiently far under $\Qt_X$ so that 
$\E_{\Qt}(h'; h) \gtrsim \epsilon$, the lower-bound construction is such that 
\begin{align}
{\rm KL}(\Pi_h \| \Pi_{h'}) \lesssim \sum_{t=1}^{N+1} n_t \epsilon^{\rho_t} \lesssim 1,\label{eq:lowerboundcriteria}
\end{align}
ensuring that $\Pi_h, \Pi_{h'}$ are hard to distinguish from finite sample. Thus, the largest $\epsilon$ satisfying the second inequality above, say $\epsilon_{N+1}$ is a minimax lower-bound. 
On the other-hand, the upper-bound (Theorem \ref{thm:upper}) relies on a uniform Bernstein's inequality that holds for 
non-identically distributed r.v.s (Lemma \ref{lem:uniform-bernstein}); in particular, by accounting for \emph{variance} in the risk, such Bernstein-type inequality allows us to extend (to the multisource setting) usual fixed-point arguments that capture the effect of the noise parameter $\beta$. Now, again for illustration, let $\beta =1$, and consider the ERM $\hat h \doteq \hat h_Z$ combining all $N+1$ datasets. Let $\E_\alpha \doteq \sum_{t=1}^{N+1} \alpha_t \E_t(\hat h)$, $\alpha_t \doteq n_t/\paren{\sum_s n_s}$, then the concentration arguments described above ensure that $\E_\alpha(\hat h) \lesssim 1/\paren{\sum_t n_t}$. Now notice that, by definition of $\rho_t$, 
$ \E_\alpha(\hat h) \gtrsim \sum_t \alpha_t \E_{\Qt}^{\rho_t}(\hat h)$, in other words, $\E_{\Qt}^{\rho_t}(\hat h)$ satisfies the second inequality in \eqref{eq:lowerboundcriteria}, and must therefore be at most of order $\epsilon_{N+1}$. This establishes the tightness of $\epsilon_{N+1}$ as a minimax rate, all that is left being to elucidate its exact form in terms of sample sizes. Similar, but somewhat more involved arguments apply for general $\beta$, though in that case we 
find that pooling all of the data does not 
suffice to achieve the minimax rate.

Notice that the rates of Theorem \ref{thm:minimax} immediately imply minimax rates for multitask under the assumption (A1) of sharing a same $\hstar$ (with appropriate $\rho_t$'s w.r.t.\ any target $\Qt \doteq P_s, s\in [N+1]$). It is then natural to ask whether the minimax rate for various targets $P_s$ might be achieved by the same algorithm, i.e., the same aggregation of tasks, in light of a common approach in the literature (and practice) of optimizing for a single classifier that does well simultaneously on all tasks. We show that even when all $\hstar$'s are the same, the optimal aggregation might differ across targets $P_s$, simply due to the inherent assymmetry of transfer. We have the following theorem, proved in Appendix \ref{app:asymmetry}.

\begin{theorem}[Target affects aggregation]\label{theo:asymmetry}
Set $N=1$. 
There exists $P, \Qt$ satisfying a Bernstein class condition with parameters $(C_\beta, \beta)$ for some {$0 \leq \beta < 1$}, and sharing the same $\hstar = \hstar_P = \hstar_\Qt$ such that the following holds. 
Consider a multisample $Z = \braces{Z_P, Z_\Qt}$ consisting of independent datasets $Z_P \sim P^{n_P}, Z_\Qt \sim \Qt^{n_\Qt}$. 

Let $\hat h_Z, \hat h_{Z_P}, \hat h_{Z_\Qt}$ denote the ERMs over $Z$, $Z_P$, $Z_\Qt$ respectively.  
Suppose $1 \leq n_\Qt^2 \leq \frac{1}{8} n_P^{(2- 2\beta)/(2-\beta)}.$ Then 
\begin{align*}
\expec \brackets{\E_{\Qt}(\hat h_{Z_P})} \land \expec \brackets{\E_{\Qt}(\hat h_Z)} &\geq \frac{1}{4}, \quad  \text{while } \expec\brackets {\E_{\Qt}(\hat h_{Z_\Qt})} \lesssim n_\Qt^{-1/(2-\beta)}; \\
\text{ however, }\,  \expec\brackets{ \E_{P}(\hat h_{Z_P})} \lor \expec \brackets{\E_{P}(\hat h_{Z})} &\lesssim n_P^{-1/(2-\beta)}.
\end{align*}
\end{theorem}

\begin{remark}[Suboptimality of \emph{pooling}]
A common practice is to \emph{pool} all datasets together and return an ERM as in $\hat h_Z$. We see from the above result that this might be optimal for some targets while suboptimal for other targets. However, pooling is near optimal (simultaneously for all targets $P_s$) whenever $\beta =1$, as discussed in Section \ref{sec:adaptivityoverview} below. 
\end{remark}

\subsection{Some Regimes of (Semi) Adaptivity}\label{sec:adaptivityoverview}
It is natural to ask whether the above minimax rates for $\M$ are attainable by \emph{adaptive procedures}, i.e., a reasonable procedure with no access to prior information on (the parameters) of $\M$, but only access to a multisample $Z\sim \Pi$ for some unknown $\Pi \in \M$. As we will see in Section \ref{sec:impossibilityoverview}, this is not possible in general, i.e., outside of the regimes considered here. Our work however leaves open the existence of more refined regimes of adaptivity.

{\bf $\bullet$ Low Noise $\beta = 1$.} 
To start, when the Bernstein class parameter $\beta = 1$ (which would often be a priori unknown to the learner),  \emph{pooling} of all datasets is near minimax optimal as stated in the next result. This corresponds {to low noise situations, e.g., so-called \emph{Massart's noise} (where $\prob(Y \neq \hstar(X)|X) \leq (1/2) - \tau$, for some $\tau > 0$), including the \emph{realizable case} (where $Y= \hstar(X)$ deterministically).}
Note that $\rho_t$'s are nonetheless nontrivial (see examples of Section \ref{sec:transfer}), however the distributions $P_t$'s are then sufficiently related that their datasets are mutually valuable.

\begin{theorem}[Pooling under low noise]
\label{thm:beta-1-union} 
Suppose $\beta = 1$. Consider any $\Pi \in \M$ and let $\hat h_Z$ denote the ERM over $Z\sim \Pi$.
Let $\delta \in (0, 1)$. There exists a universal constant $c>0$, such that, 
with probability at least $1-\delta$, 
\begin{equation*}
\E_{\Qt}(\hat{h}_{Z}) 
\leq \min_{t \in [N+1]} C_{\rho} \left( c\ C_{\beta} \frac{\vc \log\!\left(\frac{1}{\vc} \sum_{s=1}^{N+1} n_{s} \right) + \log(1/\delta)}{\sum_{s=1}^{t} n_{(s)}} \right)^{1/\brho{t}}. 
\end{equation*}
\end{theorem}

The theorem is proven in Section \ref{sec:semiadaptive1}.
We also state a 
general bound on $\E_{\Qt}(\hat{h}_{Z})$, 
holding for any $\beta$, 
in Corollary~\ref{cor:general-pooling} of Appendix~\ref{app:pooling-median};
the implied rates are not always 
optimal, though interestingly 
they are near-optimal in the 
case that 
$\sum_{t=1}^{t^*} n_{(t)} \propto \sum_{t=1}^{N} n_{t}$, where 
$t^*$ is the minimizer of the 
r.h.s.\ in Theorem~\ref{thm:minimax}.
We note that, unlike in the oracle upper-bounds of Theorem \ref{thm:upper}, the logarithmic term in the above result is in terms of the entire sample size, rather than the sample size at which the minimizer in $t\in [N+1]$ is attained.

{\bf $\bullet$ Available ranking information.} Now, assume that on top of $Z\sim \Pi \in \M$, we have access to 
ranking information $\rho_{(1)}\leq \rho_{(2)} \leq \ldots \leq \rho_{(N+1)}$, but no additional information on $\M$. Namely, $C_\beta, \beta$, and the actual values of $\rho_t, t\in [N]$ are unknown to the learner. We show that, in this case, a simple rank-based procedure achieves the minimax rates of Theorem \ref{thm:minimax}, without knowledge of the additional distributional parameters.  

Define 
$\Abound(m,\delta) \doteq \frac{\vc}{m}\Log\!\left(\frac{m}{\vc}\right)+\frac{1}{m}\Log\!\left(\frac{1}{\delta}\right)$, for $\delta \in (0,1)$ and $m \in \nats$. Let $\NN{t} \doteq \sum\limits_{s=1}^{t} n_{(s)}$
for each $t \in [N+1]$, ansd recall that $\hat h_{Z^{(t)}}$ denotes the ERM over the aggregate sample $Z^{(t)}$.

{\begin{quote}\emph{Rank-based Procedure $\hat h$:} let $\delta_t \doteq \delta / (6t^2)$, and 
$\Aconst$ as in Lemma~\ref{lem:uniform-bernstein}. For any $t \in [N+1]$, define: 
    \begin{equation}
\label{eqn:rank-constraints}
\H_{(t)}\doteq \braces{h\in \H: \hat{\E}_{Z^{(t)}}(h;\hat{h}_{Z^{(t)}})
\leq \Aconst \sqrt{ \hat{\prob}_{Z^{(t)}}(h \neq \hat{h}_{Z^{(t)}}) \Abound(\NN{t},\delta_t)} + \Aconst  \Abound(\NN{t},\delta_t)},
\end{equation}
Return any $h$ in $\bigcap_{s =1}^{N+1} \H_{(s)}$, if not empty, otherwise return any $h$. 
\end{quote}}

We have the following result for this learning algorithm.

\begin{theorem}[Semi-adaptive method under ranking information]
\label{thm:semi-adaptive}
Let $\delta \in (0, 1)$.
Let $\hat{h}$ denote the above procedure, trained on $Z \sim \Pi$, for any $\Pi \in \M$. 
 There exists a universal constant {$c>0$}, such that, 
with probability at least $1-\delta$, 
\begin{equation*}
\E_{\Qt}(\hat{h}) \leq \min_{t \in [N+1]} C_{\rho} \left( c\ C_{\beta} \frac{\vc \log\!\left(\frac{1}{\vc} \sum_{s=1}^{t} n_{(s)} \right) + \log(1/\delta)}{\sum_{s=1}^{t} n_{(s)}} \right)^{1/(2-\beta)\brho{t}}.
\end{equation*}
\end{theorem}

The result is shown in Section \ref{sec:semiadaptive2} by arguing that each $h$ in $\H_{(t)}$ has ${\E_{\Qt}(h)} $ of order $\paren{\sum_{s\in [t]} n_{(s)}}^{-1/(2-\beta)\brho{t}}$, so the minimizer is attained whenever $\hat t = N+1$ (which is probable, as $\hstar$ is likely to be in each $\H_{(t)}$).

{Recalling our earlier discussion of Section \ref{sec:introduction}, the above result applies to the case of \emph{drifting distributions} by letting $n_t=1$,
and $(t) = N+2-t$, i.e., the $t^{{\rm th}}$ previous example 
is considered the $t^{{\rm th}}$ most-relevant to the target $\Qt$.
In this case, in contrast to the lower bounds proven by 
\cite{bartlett1992learning}, the above result of Theorem~\ref{thm:semi-adaptive} reveals situations where 
the risk at time $N$ approaches $0$ 
as $N \to \infty$, even though the total-variation distance 
between adjacent distributions may be bounded away from zero (see examples of Section \ref{sec:transfer}).
In other words, by constraining the sequence of $P_t$ distributions 
by the sequence of $\rho_{t}$ values, we can describe scenarios 
where the traditional upper bound analysis of drifting distributions 
from \cite{bartlett1992learning, ben1989parametrization, barve1997complexity} can be improved.}

\begin{remark}[approximate ranking]
Theorem~\ref{thm:semi-adaptive},
and its proof, also have 
implications for 
scenarios where we may only 
have access to an 
\emph{approximate} ranking: 
that is, where 
the ranking indices $(t)$ 
don't strictly order the 
tasks by their respective minimum 
valid $\rho_{t}$ 
values.  For instance, 
one immediate observation 
is that, since Definition~\ref{def:rho} 
does not require $\rho_t$ 
to be minimal, any 
larger value of $\rho_t$ 
would also be valid; 
therefore, for any 
permutation 
$\sigma : [N+1] \to [N+1]$,
there always 
exists \emph{some} 
valid choices of $\rho_t$ 
values so that the 
sequence $\rho_{\sigma(t)}$ 
is non-decreasing, 
and hence we can define
$(t) = \sigma(t)$,
so that Theorem~\ref{thm:semi-adaptive} holds 
(with these $\rho_t$ 
values) for 
the rank-based procedure 
applied with this $\sigma(t)$ 
ordering of the tasks. 
For instance, this means 
that if we use an 
ordering $\sigma(t)$ that 
only swaps the order of 
some $\rho_{(t)}$ values that are 
within some $\epsilon$ of each 
other, then the result in the 
theorem remains valid aside 
from replacing $\brho{t}$ with 
$\brho{t}+\epsilon$.
A second observation is that, 
upon inspecting the proof, 
it is clear that 
the result is only truly 
sensitive to the ranking 
relative to the index $t^*$
achieving the minimum 
in the bound: that is, 
if some indices $t,t'$ with 
$(t) < (t^*)$ and $(t') < (t^*)$ 
are incorrectly ordered, 
while still both ranked before 
$t^*$ (or likewise for 
indices with $(t) > (t^*)$ 
and $(t') > (t^*)$), 
the result remains valid 
as stated nonetheless.
\end{remark}

\subsection{Impossibility of Adaptivity in General Regimes}\label{sec:impossibilityoverview}
Adaptivity in general is not possible even though the rates of Theorem \ref{thm:minimax} appear to reduce the exponential search space over aggregations to a more manageable one based on rankings of data. As seen in the previous section, easier settings such as ones with ranking information, or low noise, allow adaptivity to the remaining unknown distributional parameters. The result below states that, outside of these regimes, no algorithm can guarantee a rate better than a dependence on the number of samples {$n_\Qt$ from the target task}, even when a semi-adaptive procedure using ranking information can achieve any desired rate $\epsilon$.

\begin{theorem}[Impossibility of adaptivity] \label{thm:nonadaptivity}
Pick any \sk{$0\leq \beta < 1$, 
$C_\beta \geq 2$}, and let ${1\leq n < 2/\beta - 1}$, $n_\Qt \geq 0$, and $C_\rho =3$. 
Pick any $\epsilon >0$. The following holds for $N$ sufficiently large.

$\bullet$ Let $\hat h$ denote any multisource learner with no knowledge of $\M$. There exists a multisource class $\M$ with parameters $n_{N+1} = n_\Qt$, ${n}_{t} = n, \forall t\in [N]$, and all other parameters satisfying the above, such that 
\begin{align*}
    \sup_{\Pi \in \M} \expec_\Pi\brackets{\E_{\Qt}(\hat h)} \geq c\cdot \left(1\land n_\Qt^{-1/(2-\beta)}\right)\text{ for a universal constant } c>0. 
\end{align*}

$\bullet$ On the other hand, 
there exists a semi-adaptive classifier 
$\tilde h$, which, given data $Z\sim \Pi$, along with a ranking $\rho_{(1)} \leq \rho_{(2)} \ldots \leq \rho_{(N+1)}$, but no other information on the above $\M$, achieves the rate 
$$  \sup_{\Pi \in \M} \expec_{\Pi}  \brackets {\E_\Qt  (\tilde h)} \leq \epsilon.$$
\end{theorem}

The first part of the result follows from Theorem \ref{thm:nonadaptivityfull} of Section \ref{sec:nonadaptive}, while the second part follows from both Theorem \ref{thm:nonadaptivityfull} and  Theorem \ref{thm:semi-adaptive} of Section \ref{sec:nonadaptive}. The main idea of the proof is to inject enough randomness into the choice of ranking, while at the same time allowing \emph{bad} datasets from distributions with large $\rho_t$ -- which would force a wrong choice of $\hstar$ -- to appear as benign as \emph{good} samples from distributions with small $\rho_t =1$. {Hence, we let $N$ large enough in our constructions so that the bulk of {bad} datasets would significantly \emph{overwhelm} the information from \emph{good} datasets. }

As a technical point, \emph{no knowledge of $\M$} simply means that a minimax analysis is performed over a larger family containing such $\M$'s, indexed over choices of ranking each corresponding to a fixed $\M$. 

{Finally, we note that the result leaves open the possibility of adaptivity under further distributional restrictions on $\M$, for instance requiring that the number of samples $n$ per source task be large w.r.t. other parameters such as $\beta$ and $N$; although this remains unclear, large values of $n$ could perhaps yield sufficient information to compare and rank tasks.}

{Another possible restriction towards adaptivity is to require that a large proportion of the samples are from \emph{good} datasets w.r.t. $N+1$}. {In particular, we show 
in Theorem~\ref{thm:pooling-median} of Appendix~\ref{app:pooling-median} that 
the ERM $\hat{h}_{Z}$ which pools all datasets 
achieves a target excess risk $\E_{\Qt}(\hat{h}_{Z})$ depending on the 
(weighted) \emph{median} of the $\rho_t$ values 
(or more generally, any \emph{quantile}); in other words as long as a constant fraction of 
all datapoints (pooled from all datasets) are from tasks with relatively small $\rho_t$, the bound will be small.}
However, this is not a safe algorithm in general, as per Theorem~\ref{theo:asymmetry}.

\section{Lower Bound Analysis}\label{sec:lowerbound}
\label{sec:lower}

\begin{theorem}
\label{thm:lower}
Suppose $|\H| \geq 3$.
If every $\rho_t \geq 1$, then 
for any learning rule $\hat{h}$, there exists 
$\Pi \in \M$ 
such that, with probability at least $1/50$, 
\begin{equation*}
\E_{\Qt}(\hat{h}) > c_1 \min\limits_{t \in [N+1]} \left( \frac{c_2 \vc}{\left( \sum_{s=1}^{t} n_{(s)} \right) \Log^2 \!\left( \sum_{s=1}^{t} n_{(s)} \right) } \right)^{1 / (2-\beta) \brho{t}}
\end{equation*}
for numerical constants $c_1,c_2 > 0$.
\end{theorem}
\begin{proof}
We will prove this result with $C_{\beta} = C_{\rho} = 2$;
the general cases $C_{\beta} \geq 2$ and $C_{\rho} \geq 2$ 
follow immediately (since distributions satisfying 
$C_{\beta} = C_{\rho} = 2$ also satisfy the respective 
conditions for any $C_{\beta} \geq 2$ and $C_{\rho} \geq 2$).
As in the proof of Theorem~\ref{thm:upper}, 
for each $t \in [N+1]$, define 
$\NN{t} \doteq \sum\limits_{s=1}^{t} n_{(s)}$.
We prove the theorem using a standard approach to lower bounds by the 
probabilistic method.
Let $x_{0},x_{1},\ldots,x_{d}$ be a sequence in $\X$ such that, 
$\exists y_{0} \in \Y$ for which every $y_1,\ldots,y_d \in \Y$ 
can be realized by $h(x_1),\ldots,h(x_d)$ for some $h \in \H$ with $h(x_0)=y_0$.
If $\vc = 1$, 
we let $d=1$, and we can find such a sequence since $|\H| \geq 3$; 
if $\vc \geq 2$,
then we let $d=\vc-1$ and any shattered sequence of $\vc$ 
points will suffice.
We now specify a family of $2^d$ distributions, as follows.

Fix $\epsilon \in (0,1)$.
For each $\sigma \in \{-1,1\}^{d}$, for each $t \in [N+1]$, 
let $P_{t}^{\sigma}(\{x_0,y_0\}) = 1 - \epsilon^{\rho_{t} \beta}$, 
$P_{t}^{\sigma}(\{x_0,-y_0\}) = 0$,
and for each $i \in [d]$ and $y \in \{-1,1\}$ let 
$P_{t}^{\sigma}(\{x_i,y\}) = \frac{1}{d} \epsilon^{\rho_{t}\beta}\left( \frac{1}{2} + \frac{y\sigma_i}{2} \epsilon^{\rho_{t}(1-\beta)} \right)$.
In other words, the marginal probability of each $x_i$ is $\frac{1}{d} \epsilon^{\rho_{t}\beta}$ 
and the conditional probability of label $1$ given $x_i$ is 
$\frac{1}{2} + \frac{\sigma_i}{2}\epsilon^{\rho_{t}(1-\beta)}$.

We first verify that $\Pi_{\sigma} = \prod\limits_{t \in [N+1]} P_{t}^{\sigma}$ 
is in $\M$.  For every $\sigma$ and every $t$, 
the Bayes optimal label at $x_0$ is always $y_0$, 
and the Bayes optimal label at $x_i$ ($i \in [d]$) is always $\sigma_i$; 
since this is the same for every $t$, the classifier $\hstar_{\sigma} \doteq h_{P_{N+1}^{\sigma}}^{*}$ 
is optimal under every $P_{t}^{\sigma}$.
Next, we check the Bernstein class condition.
Define $\ell(h,\sigma) \doteq |\{ i \in [d] : h(x_i) \neq \sigma_i\}|$.
For any $t \in [N+1]$ and $h \in \H$, 
we have  
\begin{equation*}
P_{t}^{\sigma}( h \neq \hstar_{\sigma} ) 
= (1 - \epsilon^{\rho_{t}\beta}) \ind{ h(x_0) \neq y_0 } 
+ \frac{\ell(h,\sigma)}{d} \epsilon^{\rho_{t}\beta}
\end{equation*}
while
\begin{equation*}
\E_{P_t^{\sigma}}(h) = 
(1 - \epsilon^{\rho_{t}\beta}) \ind{ h(x_0) \neq y_0 } 
+ \frac{\ell(h,\sigma)}{d} \epsilon^{\rho_{t}}.
\end{equation*}
Thus,
\begin{equation*}
2 \E_{P_t^{\sigma}}^{\beta}(h) \geq 
2 \max\!\left\{ (1 - \epsilon^{\rho_{t}\beta})^{\beta} \ind{ h(x_0) \neq y_0 }, \epsilon^{\rho_{t}\beta} \left(\frac{\ell(h,\sigma)}{d}\right)^{\beta} \right\}
\geq 
P_{t}^{\sigma}( h \neq \hstar_{\sigma} ).
\end{equation*}

Finally, we verify the $\rho_t$ values are satisfied.
First note that any $t$ with $\rho_t = \infty$ trivially satisfies the condition. 
Denote by $\Qt^{\sigma}$ the distribution $P_{N+1}^{\sigma}$.
Then for each $t \in [N]$ with $1 \leq \rho_t < \infty$, and each $h \in \H$, 
\begin{align*}
\E_{\Qt^{\sigma}}(h) & = 
\left( \left( (1 - \epsilon^{\beta}) \ind{ h(x_0) \neq y_0 } 
+ \epsilon \frac{\ell(h,\sigma)}{d} \right)^{\rho_{t}} \right)^{1/\rho_{t}}
\\ & \leq 2 \left( (1 - \epsilon^{\beta})^{\rho_{t}} \ind{ h(x_0) \neq y_0 }
+ \left( \epsilon \frac{\ell(h,\sigma)}{d}\right)^{\rho_{t}} \right)^{1/\rho_{t}}
\leq 2 \E_{P_t^{\sigma}}^{1/\rho_{t}}(h),
\end{align*}
where the final inequality follows from $\rho_t \geq 1$ 
and a bit of calculus to verify that 
$(1-x)^{\rho_{t}} \leq 1-x^{\rho_{t}}$ for all $x \in [0,1]$.

Now the Varshamov-Gilbert bound
(see Proposition~\ref{lem:VGBound} of Appendix~\ref{app:aux-lemmas}) 
implies there exists a subset 
$\{\sigma^{0},\sigma^{1},\ldots,\sigma^{M}\} \subset \{-1,1\}^d$ 
with $M \geq 2^{d/8}$, 
$\sigma^{0} = (1,1,\ldots,1)$, 
and every distinct $i,j$ have 
$\sum_{t} \ind{\sigma^{i}_{t} \neq \sigma^{j}_{t}} \geq \frac{d}{8}$.
Furthermore, for any $i \neq 0$, 
\begin{align*}
{\rm KL}(\Pi_{\sigma^i}\| \Pi_{\sigma^0}) 
& = \sum_{t \in [N+1]} n_{t} {\rm KL}(P_{t}^{\sigma^i} \| P_{t}^{\sigma^0})
= \sum_{t \in [N+1]} n_{t} \epsilon^{\rho_{t}\beta} \frac{1}{d} \sum_{j \in [d]} \ind{ \sigma^{i}_{j} \neq 1 } {\rm KL}( \frac{1}{2} - \frac{1}{2} \epsilon^{\rho_{t}(1-\beta)} \| \frac{1}{2} + \frac{1}{2} \epsilon^{\rho_{t}(1-\beta)} )
\\ & \leq c_3 \sum_{t \in [N+1]} n_{t} \epsilon^{\rho_{t}\beta} \frac{1}{d} \sum_{j \in [d]} \ind{ \sigma^{i}_{j} \neq 1 } \epsilon^{2\rho_{t}(1-\beta)} 
\leq c_3 \sum_{t \in [N+1]} n_{t} \epsilon^{\rho_{t}(2-\beta)}
\end{align*}
for a numerical constant $c_3$, 
where we have used a quadratic approximation of KL 
divergence between Bernoullis (Lemma~\ref{lem:klbound} of Appendix~\ref{app:aux-lemmas}).
Consider any choice of $\epsilon > 0$ making this last expression 
less than $\frac{d}{64}$.

Since $\frac{d}{64} \leq (1/8)\log(M)$,
Theorem 2.5 of \citep*{tsybakov2009introduction} 
(see Proposition~\ref{prop:tsy25} of Appendix~\ref{app:aux-lemmas}) 
implies that for any (possibly randomized) 
estimator $\hat{\sigma} : (\X \times \Y)^{\sum_t {n}_{t}} \to \{-1,1\}^d$, 
there exists a choice of $\sigma$ such that 
a sample $Z \sim \Pi_{\sigma}$ will have 
$| \{ i : \hat{\sigma}(Z)_{i} \neq \sigma_i \} | \geq \frac{d}{16}$
with probability at least $\frac{1}{50}$.

In particular, for any learning algorithm $\hat{h}$, 
we can take $\hat{\sigma} = (\hat{h}(x_1),\ldots,\hat{h}(x_d))$, 
and this result implies that there is a choice of $\sigma$ 
such that, for $\hat{h}$ trained on $Z \sim \Pi_{\sigma}$, 
with probability at least $\frac{1}{50}$, 
there are at least $\frac{d}{16}$ points $x_i$ with 
$\hat{h}(x_i) \neq \sigma_{i} = \hstar_{\sigma}(x_i)$, 
which implies 
$\E_{\Qt^{\sigma}}(\hat{h}) \geq \frac{\epsilon}{16}$.

It remains only to identify an explicit value of $\epsilon$ 
for which 
$\sum_{t \in [N+1]} n_{t} \epsilon^{\rho_{t}(2-\beta)} 
\leq \frac{d}{64}$.

In particular, fix any positive function 
$\phi(m)$ such that $\sum_{m=1}^{\infty} \phi(m)^{-1} \leq c d$ 
for a finite numerical constant $c$: 
for instance, $\phi(m) = \frac{1}{d} m \log^2 m$ 
will suffice to prove the theorem.
Then consider 
\begin{equation*}
\epsilon = \min_{t \in [N+1]} \left( c_2^{-1} \phi(\NN{t})\right)^{- 1 / (2-\beta) \rho_{(t)}}
\end{equation*}
for a numerical constant $c_2 \in (0,1)$.
Denote by $t_{*}$ the value of $t$ obtaining the minimum in this expression.
Then note that every other $t \neq t_{*}$ has 
\begin{equation*}
\left( c_2^{-1} \phi(\NN{t}) \right)^{-1 / \rho_{(t)}}
\geq \left( c_2^{-1} \phi(\NN{t_{*}})\right)^{- 1 / \rho_{(t_{*})}},
\end{equation*}
which implies 
\begin{equation*}
\frac{\rho_{(t)}}{\rho_{(t_{*})}} \geq \frac{\ln\left( c_2^{-1} \phi(\NN{t}) \right)}{\ln\left( c_2^{-1} \phi(\NN{t_{*}}) \right)},
\end{equation*}
so that
\begin{equation*}
\left( c_2^{-1} \phi(\NN{t_{*}}) \right)^{-\rho_{(t)}/\rho_{(t_{*})}}
\leq \left( c_2^{-1} \phi(\NN{t_{*}}) \right)^{-\frac{\ln\left( c_2^{-1} \phi(\NN{t}) \right)}{\ln\left( c_2^{-1} \phi(\NN{t_{*}}) \right)}}
= \left( c_2^{-1} \phi(\NN{t}) \right)^{-1}.
\end{equation*}
Thus, we have 
\begin{equation*}
\sum_{t \in [N+1]} n_{t} \epsilon^{\rho_{t}(2-\beta)} 
\leq \sum_{t \in [N+1]} n_{(t)} \left( c_2^{-1} \phi(\NN{t}) \right)^{-1}
\leq c_2 \sum_{m = 1}^{\infty} \frac{1}{\phi(m)}
\leq c_2 c d
\end{equation*}
for a finite numerical constant $c$.
Thus, choosing any $c_2 < 1 / (64 c)$, 
we have 
$\sum_{t \in [N+1]} n_{t} \epsilon^{\rho_{t}(2-\beta)} < \frac{d}{64}$, as desired.

Altogether, we have that for any learning rule $\hat{h}$, 
there exists $\Pi \in \M$ such that if $\hat{h}$ is trained on $Z \sim \Pi$, 
then with probability at least $1/50$, 
\begin{equation*}
\E_{\Qt}(\hat{h}) \geq \frac{\epsilon}{16} 
= \frac{1}{16} \min_{t \in [N+1]} \left( \frac{c_2}{\phi(\NN{t})} \right)^{1 / (2-\beta) \rho_{(t)}}.
\end{equation*}
In particular, since we always have $\brho{t} \leq \rho_{(t)}$, 
the theorem immediately follows.
\end{proof}

\begin{remark}
We note that it is clear from the proof that 
the $\left( \sum_{s=1}^{t} n_{(s)} \right) \Log^2 \!\left( \sum_{s=1}^{t} n_{(s)} \right)$ denominator in the lower bound is not 
strictly optimal.  It can 
be replaced by any function $\phi'\!\left( \sum_{s=1}^{t} n_{(s)} \right)$ satisfying $\sum_{m=1}^{\infty} \phi'(m)^{-1} < \infty$:
for instance, $\phi'(m) = m\Log(m) (\Log \Log(m))^2$. 
\end{remark}

\begin{remark}
We also note that there exist classes $\H$ for which 
the lower bound can be extended to the full range of $\{\rho_t\}$ (i.e., any values $\rho_t > 0$).
Specifically, in this general case, 
if there exist two points $x_0,x_1 \in \X$ such that 
\emph{all} $h \in \H$ agree on $h(x_0)$ while $\exists h,h' \in \H$ with $h(x_1) \neq h'(x_1)$, 
then by the same construction (with $d=1$) used in the proof we would have, 
with probability at least $1/50$, 
\begin{equation*}
\E_{\Qt}(\hat{h}) > c_1 \min\limits_{t \in [N+1]} \left( \frac{c_2}{\left( \sum_{s=1}^{t} n_{(s)} \right) \Log^2 \!\left( \sum_{s=1}^{t} n_{(s)} \right) } \right)^{1 / (2-\beta) \brho{t}}.
\end{equation*}
\end{remark}

\section{Upper Bound Analysis}\label{sec:upperbound}
\label{sec:upper}

We will in fact establish the upper bound as a bound holding with 
high probability $1-\delta$, for any $\delta \in (0,1)$.
Throughout this subsection, let 
$\M = \M\paren{C_\rho, \braces{{\rho_t}}_{t\in [N]}, \braces{{n}_{t}}_{t \in [N+1]}, C_\beta, \beta}$, for any admissible values of the parameters.  
Let $$t^{*} \doteq \argmin\limits_{t \in [N+1]} C_{\rho} \left( 2^{10} \Aconst^{4} C_{\beta} \frac{\vc \Log \!\left(\frac{1}{\vc} \sum_{s=1}^{t} n_{(s)} \right)  + \log(1/\delta)}{\sum_{s=1}^{t} n_{(s)} } \right)^{1/(2-\beta)\brho{t}}, $$
for $\Aconst$ as in Lemma~\ref{lem:uniform-bernstein} below, 
and for $\delta \in (0,1)$. The oracle procedure just returns $\hat{h}_{Z^{(t^*)}}$, the ERM 
over $Z^{(t^*)}$.

We have the following theorem.

\begin{theorem}
\label{thm:upper}
For any 
$\Pi \in \M$, 
for any $\delta \in (0,1)$, with probability at least $1-\delta$, we have 
\begin{equation*}
\E_{Q_t}(\hat{h}_{Z^{(t^*)}}) \leq \min_{t \in [N+1]} C_{\rho} \left( C \frac{\vc \Log\left(\frac{1}{\vc}\sum_{s =1}^t n_{(s)}\right) + \Log(1/\delta)}{\sum_{s = 1}^{t} n_{(s)}} \right)^{1/ (2-\beta) \brho{t}}
\end{equation*}
for a constant $C = 2^{10} \Aconst^4 C_{\beta}$, 
where $\Aconst$ is a numerical constant from Lemma~\ref{lem:uniform-bernstein} below.
\end{theorem}

The proof will rely on the following lemma: a uniform Bernstein inequality
for independent but non-identically distributed data.  
Results of this type are well known for i.i.d.\ data  
(see e.g., \cite{koltchinskii:06}). 
For completeness, we include a proof of the extension 
to non-identically distributed data 
in Appendix~\ref{app:uniform-bernstein}.
The main technical modification of the proof, compared to the i.i.d.\ case,  
is employing a generalization of Bousquet's inequality 
for non-identical data, due to \cite{klein2005concentration}.

\begin{lemma}
\label{lem:uniform-bernstein}
(Uniform Bernstein inequality for non-identical distributions)~~
For any $m \in \nats$ 
and $\delta \in (0,1)$,
define 
$\varepsilon(m,\delta) \doteq \frac{\vc}{m}\Log\!\left(\frac{m}{\vc}\right)+\frac{1}{m}\Log\!\left(\frac{1}{\delta}\right)$
and
let $S=\{(X_1,Y_1),\ldots,(X_m,Y_m)\}$ 
be independent samples.
With probability at least $1-\delta$,
$\forall h,h' \in \H$, 
\begin{equation}
\label{eqn:bernstein-1}
\EE\!\left[ \hat{\E}_{S}(h;h')\right]
\leq \hat{\E}_{S}(h;h')  + \Aconst\sqrt{ \min\!\left\{\EE\!\left[\hat{\prob}_{S}(h \neq h') \right],\hat{\prob}_{S}(h \neq h') \right\} \Abound(m,\delta)} + \Aconst \Abound(m,\delta),
\end{equation}
and
\begin{equation}
\label{eqn:bernstein-2}
\frac{1}{2}\EE\!\left[\hat{\prob}_{S}( h \neq h' ) \right] - \Aconst \Abound(m,\delta)
\leq \hat{\prob}_{S}( h \neq h' ) 
\leq 2 \EE\!\left[ \hat{\prob}_{S}( h \neq h' ) \right] + \Aconst \Abound(m,\delta),
\end{equation}
for a universal numerical constant $\Aconst \in (0,\infty)$.
\end{lemma}

In particular, we will use this lemma via the 
following implication.

\begin{lemma}
\label{lem:Ealpha-bound}
For any $\Pi$ as in Theorem~\ref{thm:upper},
for any $I \subseteq [N+1]$, 
letting $\NNI{I} \doteq \sum\limits_{t \in I} n_{t}$
and $\bar{P}_I \doteq \NNI{I}^{-1} \sum_{t \in I} n_{t} P_t$,
for any $\delta \in (0,1)$, 
on the event (of probability at least $1-\delta$) 
from Lemma~\ref{lem:uniform-bernstein}
(for $S = Z^{I}$ there), 
for any $h \in \H$ satisfying 
\begin{equation}
\label{eqn:general-constraint}
\hat{\E}_{Z^{I}}(h; \hat{h}_{Z^{I}}) 
\leq \Aconst\sqrt{ \hat{\prob}_{Z^{I}}(h \neq \hat{h}_{Z^{I}}) \Abound(\NNI{I},\delta)} + \Aconst \Abound(\NNI{I},\delta),
\end{equation}
it holds that
\begin{equation*}
\E_{\bar{P}_I}(h) 
\leq 32 \Aconst^2 \left( C_{\beta} \Abound(\NNI{I},\delta) \right)^{1/(2-\beta)}.
\end{equation*}
\end{lemma}
\begin{proof}
On the event from Lemma~\ref{lem:uniform-bernstein} 
for $S = Z^{I}$, 
it holds that $\forall h,h' \in \H$, 
(since $\E_{\bar{P}_I}(h;h')=\EE[\hat{\E}_{Z^I}(h;h')]$) 
\begin{equation}
\label{eqn:alpha-bernstein-1}
\E_{\bar{P}_I}(h;h') 
\leq \hat{\E}_{Z^{I}}(h;h')
 + \Aconst \sqrt{\min\!\left\{\bar{P}_{I}( h \neq h' ), \hat{\prob}_{Z^{I}}( h \neq h' ) \right\} \Abound(\NNI{I},\delta) } 
+ \Aconst \Abound(\NNI{I},\delta)
\end{equation}
and
\begin{equation}
\label{eqn:alpha-bernstein-2}
\hat{\prob}_{Z^{I}}( h \neq h' ) 
\leq 2 \bar{P}_{I}( h \neq h' ) + \Aconst \Abound(\NNI{I},\delta).
\end{equation}
Suppose this event occurs.

By the Bernstein class condition and Jensen's 
inequality, for any $h \in \H$ we have
\begin{equation}
\label{eqn:alpha-beta-relax}
\bar{P}_{I}( h \neq \hstar ) 
\leq C_{\beta} \NNI{I}^{-1} \sum_{t \in I} n_{t} \E_{P_t}^{\beta}(h)
\leq C_{\beta} \left( \E_{\bar{P}_{I}}(h) \right)^{\beta}.
\end{equation}
Furthermore, applying \eqref{eqn:alpha-bernstein-1} 
with $h = \hat{h}_{Z^{I}}$ and $h' = \hstar$
reveals that 
\begin{equation*}
\E_{\bar{P}_{I}}(\hat{h}_{Z^{I}}) 
\leq \Aconst \sqrt{\bar{P}_{I}( \hat{h}_{Z^{I}} \neq \hstar ) \Abound(\NNI{I},\delta) } 
+ \Aconst \Abound(\NNI{I},\delta).
\end{equation*}
Combining this with \eqref{eqn:alpha-beta-relax} (for $h = \hat{h}_{Z^{I}}$) implies that 
\begin{equation}
\E_{\bar{P}_{I}}(\hat{h}_{Z^{I}}) 
\leq \Aconst \sqrt{ C_{\beta} \left( \E_{\bar{P}_{I}}(\hat{h}_{Z^{I}}) \right)^{\beta} \Abound(\NNI{I},\delta) } 
+ \Aconst \Abound(\NNI{I},\delta),
\end{equation} 
which implies 
\begin{equation}
\label{eqn:alpha-hstar-bound-1}
\E_{\bar{P}_{I}}(\hat{h}_{Z^{I}}) 
\leq 2\Aconst \left( C_{\beta} \Abound(\NNI{I},\delta) \right)^{1/(2-\beta)}.
\end{equation}

Next, applying \eqref{eqn:alpha-bernstein-1} 
with any $h \in \H$ satisfying 
\eqref{eqn:general-constraint} 
and with $h' = \hat{h}_{Z^{I}}$ 
implies 
\begin{equation}
\label{eqn:alpha-hstar-bound-2}
\E_{\bar{P}_I}(h;\hat{h}_{Z^{I}}) 
\leq 
2 \Aconst\sqrt{ \hat{\prob}_{Z^{I}}(h \neq \hat{h}_{Z^{I}}) \Abound(\NNI{I},\delta)} + 2\Aconst \Abound(\NNI{I},\delta),
\end{equation}
and \eqref{eqn:alpha-bernstein-2} 
implies the right hand side is at most
(after some simplifications of the resulting expression) 
\begin{equation}
\label{eqn:alpha-hstar-bound-3}
2 \Aconst\sqrt{ 2 \bar{P}_{I}(h \neq \hat{h}_{Z^{I}}) \Abound(\NNI{I},\delta)} + 2\Aconst(1 + \sqrt{\Aconst}) \Abound(\NNI{I},\delta).
\end{equation}
By the triangle inequality and 
\eqref{eqn:alpha-beta-relax}, 
together with 
\eqref{eqn:alpha-hstar-bound-1}, 
we have 
\begin{align*}
\bar{P}_{I}(h \neq \hat{h}_{Z^{I}})
\leq \bar{P}_{I}(h \neq \hstar) + \bar{P}_{I}(\hat{h}_{Z^{I}} \neq \hstar)
& \leq C_{\beta} \left( \E_{\bar{P}_{I}}(h) \right)^{\beta}
+ C_{\beta} \left( \E_{\bar{P}_{I}}(\hat{h}_{Z^{I}}) \right)^{\beta}
\\ & \leq C_{\beta} \left( \E_{\bar{P}_{I}}(h) \right)^{\beta} 
+ C_{\beta} (2\Aconst)^{\beta} \left( C_{\beta} \Abound(\NNI{I},\delta) \right)^{\beta/(2-\beta)}.
\end{align*}
Combining this with 
\eqref{eqn:alpha-hstar-bound-3} and \eqref{eqn:alpha-hstar-bound-2} 
implies (with some simplification of the 
resulting expression) 
\begin{equation*}
\E_{\bar{P}_I}(h;\hat{h}_{Z^{I}}) 
\leq 
2 \Aconst \sqrt{ 2 C_{\beta} \left( \E_{\bar{P}_{I}}(h) \right)^{\beta} \Abound(\NNI{I},\delta) }
+ 8 \Aconst^2 \left( C_{\beta} \Abound(\NNI{I},\delta) \right)^{1/(2-\beta)}.
\end{equation*}
Since $\E_{\bar{P}_I}(h) = \E_{\bar{P}_I}(h;\hat{h}_{Z^{I}}) + \E_{\bar{P}_{I}}(\hat{h}_{Z^{I}})$, 
together with \eqref{eqn:alpha-hstar-bound-1} 
this implies  
\begin{equation*}
\E_{\bar{P}_I}(h)
\leq 
2 \Aconst \sqrt{ 2 C_{\beta} \left( \E_{\bar{P}_{I}}(h) \right)^{\beta} \Abound(\NNI{I},\delta) }
+ 10 \Aconst^2 \left( C_{\beta} \Abound(\NNI{I},\delta) \right)^{1/(2-\beta)}.
\end{equation*}
In particular, this inequality immediately implies the 
claimed inequality 
in the lemma.
\end{proof}

We now present the proof of Theorem~\ref{thm:upper}.

\begin{proof}[Proof of Theorem~\ref{thm:upper}]
For each $t \in [N+1]$, define 
$\NN{t} \doteq \sum\limits_{s=1}^{t} n_{(s)}$.
For brevity, let $\bar{\rho} = \bar{\rho}_{t^*}$ 
and $\hat{h} = \hat{h}_{Z^{(t^*)}}$. 
Let $\bar{P} = \NN{t^*}^{-1} \sum_{t=1}^{t^*} n_{(t)} P_{(t)}$.
First note that 
\begin{equation*}
\E_{\bar{P}}(\hat{h}) 
\geq \frac{1}{\NN{t^*}} \sum_{t=1}^{t^*} n_{(t)} \left( C_{\rho}^{-1} \E_{\Qt}(\hat{h}) \right)^{\rho_{(t)}} 
\geq \left( C_{\rho}^{-1} \E_{\Qt}(\hat{h}) \right)^{\bar{\rho}}, 
\end{equation*}
where the final inequality is due to Jensen's inequality.
In particular, this implies
\begin{equation}
\label{eqn:upper-step-1}
\E_{\Qt}(\hat{h}) 
\leq C_{\rho} \E_{\bar{P}}^{1/\bar{\rho}}(\hat{h}),
\end{equation}
so that it suffices to upper bound the expression on the right hand side.
Toward this end, note that 
$\hat{h}$ trivially satisfies 
\eqref{eqn:general-constraint} 
for $I = \{(1),\ldots,(t^*)\}$.
Therefore, Lemma~\ref{lem:Ealpha-bound} implies 
that with probability at least $1-\delta$, 
\begin{equation*}
\E_{\bar{P}}(\hat{h}) 
\leq 32 \Aconst^2 (C_{\beta} \Abound(\NN{t^*},\delta))^{1/(2-\beta)}.
\end{equation*}
Combining this with \eqref{eqn:upper-step-1} 
immediately implies the theorem. 
\end{proof}

\begin{remark}
In particular, the upper bound for Theorem~\ref{thm:minimax} 
follows from Theorem~\ref{thm:upper} 
by plugging in 
$\delta = 1/\sum_{s=1}^{t^*} n_{(s)}$.
\end{remark}

\section{Partially Adaptive Procedures}
\label{sec:semiadaptive}

Fix 
any $N$, $C_{\rho}$, $\{\rho_t\}_{t \in [N]}$, $\{{n}_{t}\}_{t \in [N+1]}$, 
$C_{\beta}$, and $\beta \in [0,1]$, 
and let 
$\M = \M\paren{C_\rho, \braces{{\rho_t}}_{t\in [N]}, \braces{{n}_{t}}_{t \in [N+1]}, C_\beta, \beta}$. 

\subsection{Pooling Under Low Noise $\beta =1$}\label{sec:semiadaptive1}

We now present the proof of Theorem \ref{thm:beta-1-union} which states the near-optimality of \emph{pooling}, independent of the choice of target $\Qt$, whenever $\beta =1$.

\begin{proof}[Proof of Theorem \ref{thm:beta-1-union}]
For any $t \in [N+1]$, let 
$\NN{t} \doteq \sum_{s=1}^{t} n_{(s)}$ 
and 
$\bar{P}_{t} \doteq (\NN{t})^{-1}\sum_{s=1}^{t} n_{(s)} P_{(s)}$.
Suppose the event from Lemma~\ref{lem:uniform-bernstein} holds
for $Z=Z^{(N+1)}$, 
which occurs with probability at least $1-\delta$.
By Lemma~\ref{lem:Ealpha-bound}, we know 
that on this event, 
\begin{equation*}
\E_{\bar{P}_{N+1}}(\hat{h}_{Z}) 
\leq 32 \Aconst^2 C_{\beta} \cdot \Abound(\NN{N+1},\delta).
\end{equation*}
Combining this with the definition of $\rho_t$, 
we have 
\begin{equation*}
\NN{N+1}^{-1} \sum_{t \in [N+1]} {n}_{t} \left( C_{\rho}^{-1} \E_{\Qt}(\hat{h}_{Z}) \right)^{\rho_t} 
\leq 32 \Aconst^2 C_{\beta} \cdot \Abound(\NN{N+1},\delta).
\end{equation*}
Since the left hand side is monotonic in $\E_{\Qt}(\hat{h}_{Z})$, 
if we wish to bound $\E_{\Qt}(\hat{h}_{Z})$
by some value $\epsilon$, 
it suffices to take any value of $\epsilon$ such that 
\begin{equation*}
    \NN{N+1}^{-1} \sum_{t \in [N+1]} {n}_{t} \left( C_{\rho}^{-1} \epsilon \right)^{\rho_t} 
\geq 32 \Aconst^2 C_{\beta} \cdot \Abound(\NN{N+1},\delta),
\end{equation*}
or more simply
\begin{equation*}
\sum_{t \in [N+1]} {n}_{t} \left( C_{\rho}^{-1} \epsilon \right)^{\rho_t} 
\geq 32 \Aconst^2 C_{\beta} \left( \vc \log\!\left( \frac{\NN{N+1}}{\vc}  \right) + \log\!\left(\frac{1}{\delta}\right) \right).
\end{equation*}
In particular, let us take 
\begin{equation*}
\epsilon \doteq C_{\rho} \left( 32 \Aconst^2 C_{\beta} \frac{ \vc \log\!\left( \NN{N+1} / \vc \right) + \log(1/\delta)}{\NN{t^*}} \right)^{1/\brho{t^*}} 
\end{equation*}
where $t^*$ is the value of $t$ that minimizes the 
right hand side of this definition of $\epsilon$.
Then
\begin{align*}
\sum_{t \in [N+1]} {n}_{t} \left( C_{\rho}^{-1} \epsilon \right)^{\rho_t} 
&\geq \NN{t^*} \sum_{t = 1}^{t^*} \frac{n_{(t)}}{\NN{t^*}} \left( C_{\rho}^{-1} \epsilon \right)^{\rho_{(t)}}
 \geq \NN{t^*} (C_{\rho}^{-1} \epsilon)^{\brho{t^*}}
\\ & = 32 \Aconst^2 C_{\beta} \left( \vc \log\!\left( \NN{N+1} / \vc \right) + \log(1/\delta) \right).
\end{align*}
\end{proof}

\subsection{Aggregation Under Ranking Information}\label{sec:semiadaptive2}

\begin{proof}[Proof of Theorem \ref{thm:semi-adaptive}]
For any $t \in [N+1]$ let  
$\bar{P}_t = \NN{t}^{-1} \sum_{s=1}^{t} n_{(s)} P_{(s)}$.
Let $t^*$ be as in Theorem~\ref{thm:upper}, 
and by the same argument 
from the proof of Theorem~\ref{thm:upper}, 
\eqref{eqn:upper-step-1} holds for $\hat{h}$ in the present context as well (where $\bar{P}$ there is $\bar{P}_{t^*}$ 
in the present notation).
Thus, the theorem will be proven if we can bound 
$\E_{\bar{P}_{t^*}}(\hat{h})$
by $\left(c C_{\beta} \Abound(\NN{t^*},\delta_{t^*})\right)^{1/(2-\beta)}$
for some numerical constant $c$.

By a union bound, with probability at least $1 - \sum_{t=1}^{N+1} \delta_{t} \geq 1-\delta$, 
for every $t \in [N+1]$, 
the event from Lemma~\ref{lem:uniform-bernstein} 
holds for $S = Z^{(t)}$ (and with $\delta_t$ in place 
of $\delta$ there).
In particular, this implies that for every $t \in [N+1]$, 
\begin{equation*}
\hat{\E}_{Z^{(t)}}(\hstar;\hat{h}_{Z^{(t)}}) 
\leq \Aconst \sqrt{ \hat{\prob}_{Z^{(t)}}( \hstar \neq \hat{h}_{Z^{(t)}} ) \Abound(\NN{t},\delta_{t}) } + \Aconst \Abound(\NN{t},\delta_{t}),
\end{equation*}
so that there do exist classifiers $h$ in $\H$
satisfying \eqref{eqn:rank-constraints}, 
and hence $\hat{h}$ satisfies \eqref{eqn:rank-constraints}.
In particular, this implies $\hat{h}$ satisfies the 
inequality in \eqref{eqn:rank-constraints} for $t=t^*$.
By Lemma~\ref{lem:Ealpha-bound}, this implies that 
on this same event from above, it holds that 
\begin{equation*}
\E_{\bar{P}_{t^*}}(\hat{h}) 
\leq 32 \Aconst^2 (C_{\beta} \Abound(\NN{t^*},\delta_{t^*}))^{1/(2-\beta)},
\end{equation*}
which completes the proof 
(for instance, taking $c = 2^{10} \Aconst^4$).
\end{proof}

\section{Impossibility of Adaptivity over Multisource Classes $\M$} \label{sec:nonadaptive}

Theorem \ref{thm:nonadaptivity} follows as corollary to Theorem \ref{thm:nonadaptivityfull}, the main result of this section. In particular, the second part of Theorem \ref{thm:nonadaptivity} follows from the condition on $\M$ that it contains enough tasks with $\rho_t =1$, and calling on Theorem \ref{thm:semi-adaptive}.

\begin{theorem}[Impossibility of Adaptivity] \label{thm:nonadaptivityfull}
Pick $C_\rho =3$, and any \sk{$0\leq \beta < 1$}, \sk{$C_\beta \geq 2$,} 
a number of samples per source task $1\leq n < 2/\beta - 1$, and a number of target samples  $n_\Qt\geq 0$. Let the number of source tasks 
$N = N_P + N_\Q$, for $N_P, N_\Q$ as specified below. 
There exist universal constants \skr{$C_0, C_1, c>0$} such that the following holds. 

Choose any $N_\Q \geq C_0$, and suppose $N_P$ is sufficiently large so that $N_P\geq 3N_\Q$, and furthermore 
$$ N_P^{(2-(n+1)\beta)/(2-\beta)} \geq C_1 \cdot N_\Q^2\cdot 2^{15n}.$$

Let $\boldsymbol{\mathcal M}$ denote the family of multisource classes
$\M$ satisfying the above conditions, with parameters $n_t = n, \forall t \in [N]$, $n_{N+1} = n_\Qt$, and, in addition, 
such that at least $\frac{1}{2}N_\Q$ of the exponents $\braces{\rho_t}_{t\in [N]}$ are at most $1$. 

$\bullet$ 
Let $\hat h$ denote any classification procedure having access to $Z\sim \Pi, \, \Pi \in \M$, but 
without knowledge of $\M$. 
 We have:
 \sk{
\begin{align*}
    \inf_{\hat h} \sup_{{\cal M} \in \boldsymbol{\mathcal M}}\sup_{\Pi \in \M} 
    \prob_{\Pi}\paren{\E_{\Qt}(\hat h) \geq \frac{1}{4} \cdot \paren{1\land n_\Qt^{-1/(2-\beta)}}} \geq c.
\end{align*}
}

$\bullet$ {On the other hand, there exists a semi-adaptive classifier $\tilde h$, which, given data $Z\sim \Pi$, along with a ranking $\{\rho_{(t)}\}_{t \in [N+1]}$ of increasing exponents values, achieves the rate 
 $$ \sup_{{\cal M} \in \boldsymbol{\mathcal M}}\sup_{\Pi \in \M} \, \expec_{\Pi}  \brackets {\E_\Qt  (\tilde h)} \lesssim \paren{n\cdot N_\Q}^{-1/(2-\beta)}.$$}

\end{theorem}

As it turns out, the above theorem in fact holds for even smaller classes $\M$ admitting just two choices of distributions, namely $P_\sigma$ and $Q_\sigma$ (for fixed $\sigma \in \braces{\pm 1}$) below, to which sources $P_t$'s might be assigned to. In fact the result holds even if $\hat h$ has knowledge of the construction below, but of course not of $\sigma$.

\paragraph{Construction.} We build on the following distributions supported on 2 points $x_0, x_1$: here we simply assume that $\Hyp$ contains at least 2 classifiers that disagree on $x_1$ but agree on $x_0$ (this will be the case, for some choice of $x_0,x_1$, if $\Hyp$ contains at least 3 classifiers). Therefore, w.l.o.g., assume that $x_0$ has label 1. Let $n \geq 1, n_\Qt \geq 0, N_P, N_Q \geq 1$, $0\leq \beta < 1$, and define  
$\epsilon \doteq (n \cdot N_P)^{-1/(2-\beta)}$ and $\epsilon_0 \doteq 1\land n_\Qt^{-1/(2-\beta)}$. Let $\sigma\in \braces{\pm 1}$ -- which we will often abbreviate as $\pm$. In all that follows, we let $\eta_{\mu}(X)$ denote the regression function $\prob_{\mu}[Y = 1\mid  X]$ under distribution $\mu$.
\sk{
\begin{itemize} 
\item {\bf Target $\Qt_\sigma = \Qt_X \times \Qt^\sigma_{Y|X}$:} 
Let $\Qt_X(x_1) = \frac{1}{2}\cdot \epsilon_0^\beta $, $\Qt_X(x_0) = 1- \frac{1}{2}\cdot \epsilon_0^\beta$; finally $\Qt^{\sigma}_{Y|X}$ is determined by \\
$\eta_{\Qt, \sigma}(x_1) = 1/2 + \sigma\cdot {c_0}\epsilon_0^{1-\beta}$, and $\eta_{\Qt, \sigma}(x_0) = 1$, for some 
$c_0$ to be specified later. 
\item {\bf Noisy $P_\sigma = P_X \times P^\sigma_{Y | X}$:} Let $P_X(x_1) = c_1\epsilon^\beta $, $P_X(x_0) = 1- c_1\epsilon^\beta$; finally $P^\sigma_{Y | X}$ is determined by \\
$\eta_{P, \sigma}(x_1) = 1/2 + \sigma\cdot \epsilon^{1-\beta}$, and $\eta_{P, \sigma}(x_0) = 1$, 
for an appropriate constant $c_1$ specified 
later.
\item {\bf Benign $\Q_\sigma = \Q_X \times \Q^\sigma_{Y | X}$:} Let $\Q_X(x_1) = 1$; finally $\Q^\sigma_{Y|X}$ is determined by $\eta_{\Q, \sigma}(x_1) = 1/2 + \sigma/2$.
\end{itemize}
}

The above construction is such that $P_\sigma$ can be pushed far from $\Qt_\sigma$, while $\Q_\sigma$ remains close to $\Qt_\sigma$. This is formalized in the following proposition, which 
can be verified by checking 
the required inequalities 
(from Definitions~\ref{def:beta} 
and \ref{def:rho}) 
are satisfied in 
all $4$ cases of classifications 
of $x_0$, $x_1$.

\begin{proposition}[Exponents of $P$ and $\Q$ w.r.t. $\Qt$]\label{prop:rho-construct} 
$P_\sigma$ and $\Q_\sigma$ have transfer-exponents $(3, \rho_P)$ and $(3, \rho_\Q)$, respectively w.r.t. $\Qt_\sigma$, with 
\sk{$\rho_P \geq \frac{\log ( c_1^{-(2-\beta)} n\cdot N_P)}{\log (c_0^{-(2-\beta)}(1\lor n_\Qt))}$} and $\rho_Q \leq 1$. Furthermore, the 3 distributions 
$P_\sigma, \Q_\sigma, \Qt_\sigma$ satisfy the Bernstein class condition with parameters \sk{${(C_\beta, \beta)}, C_\beta = \max\{(1/2) c_0^{-\beta}, 2\}$.}
\end{proposition}

\paragraph{Approximate Multitask.} Let $N \doteq  N_P + N_Q$, $\alpha_P \doteq N_P/N$ and $\alpha_Q = N_Q/N$. Now consider the distribution 
$$\Gamma_\sigma = \paren{\alpha_P P_\sigma^n + \alpha_Q \Q_\sigma^n}^N \times \Qt_\sigma^{n_\Qt}.$$

\paragraph{Proof Strategy.} Henceforth, let $Z = \braces{Z_t}_{t = 1}^{N+1}$ denote a random draw from $\Gamma_\sigma$, where $Z_t = \braces{(X_{t, i}, Y_{t, i})}_{i = 1}^n \sim \Gamma_\sigma$ for $t \in [N]$, and $Z_{N+1}\sim \Qt_\sigma$. Much of the effort will be in showing that any learner has nontrivial error in guessing $\sigma$ from a random $Z$; we then reduce this fact to a lower-bound on the risk of an adaptive classifier that takes as input a multitask sample of the form $Z'\sim \prod_{t=1}^N P_t^n \times \Qt^n_\sigma$, where $P_t$'s denote $P_\sigma$ or $\Q_\sigma$.  

For intuition, the true label $\sigma$ is hard to guess from samples $Z_t \sim P_\sigma^n$ alone since $\eta_{P, \sigma}\approx 1/2$, but easy to guess from samples $Z_t \sim \Q^n$, each being a homogeneous vector of $n$ points $x_1$ with identical labels $Y = \sigma$. However, such homogeneous vectors can be produced by $P_\sigma^n$ also, with identical but flipped labels $Y = -\sigma$, and the product of mixtures $\Gamma_\sigma$ adds in enough randomness to make it hard to guess the source of a given vector $Z_t$. In fact, this intuition becomes evident by considering the \emph{likelihood-ratio} between $\Gamma_{+}$ and $\Gamma_{-}$ which decomposes into the contribution of homogeneous vs non-homogenous vectors. This is formalized in the following proposition. 

\begin{proposition}[Likelihood ratio and sufficient statistics]\label{prop:likelihoodRatio} Let $Z \sim \Gamma_{-}$, and let  and ${\hat N}_{+}$ and ${\hat N}_{-}$ denote the number of homogeneous vectors $Z_t$ in $Z$ with marginals at $x_1$, that is: 
$${\hat N}_{\sigma}(Z)\doteq \sum_{t\in [N]} \ind{\forall i \in [n], X_{t, i} = x_1 \, \land \,  Y_{t, i} = \sigma}.$$
Next, let ${\hat n}_{+}$ and ${\hat n}_{-}$ denote the total number of $\pm 1$ labels in $Z$, over occurrences of $x_1$. That is: 
$${\hat n}_\sigma (Z)\doteq \sum_{t\in [N], i\in [n]} \ind{X_{t, i} = x_1\, \land \, Y_{t, i} = \sigma}.$$
We then have that (where $\prob$ is under the randomness of $Z\sim \Gamma_{-}$, and we let $\Qt_\sigma(Z_{N+1}) = 1$ when $n_\Qt =0$):
\begin{align}
\prob\paren{\frac{\Gamma_{+}(Z)}{\Gamma_{-}(Z)} >  1} \geq 
\prob\paren{{\hat N}_{+}(Z) > {\hat N}_{-}(Z) \, \land \, {\hat n}_{+}(Z) \geq {\hat n}_{-}(Z)}
\cdot \prob\paren{\frac{\Qt_{+}(Z_{N+1})}{\Qt_{-}(Z_{N+1})} \geq  1}. \label{eq:nonadapt-likelihood}
\end{align}
\end{proposition}

\begin{remark}[Likelihoood Ratio and Optimal Discriminants] \label{rem:likelihood} 
Consider sampling $Z$ from the mixture 
$\frac{1}{2}\Gamma_{+} + \frac{1}{2}\Gamma_{-}$. Then the Bayes classifier (for identifying $\sigma = \pm$) returns $+1$ if $\Gamma_{+}(Z) > \Gamma_{-}(Z)$, and therefore, given the symmetry between $\Gamma_{\pm}$, has probability of misclassification at least $\prob_{\Gamma_{-}}\paren{{\Gamma_{+}(Z)} > {\Gamma_{-}(Z)}}$. 
We emphasize here that enforcing that $\Gamma_\sigma$ be defined in terms of a mixture -- rather than as product of $P_\sigma$ and $\Q_\sigma$ terms -- allows us to bound $\prob_{\Gamma_{-}}\paren{{\Gamma_{+}(Z)} > {\Gamma_{-}(Z)}}$ below as in Proposition \ref{prop:likelihoodRatio} above; otherwise this probability is always $0$ for a product distribution containing $\Q_\sigma$ since $\Q_\sigma (Z_t) = 0$ whenever some $Y_{t, i} = -\sigma$. {In fact a product distribution inherently encodes the idea that the learner \emph{knows} the positions of $P_\sigma$ and $\Q_\sigma$ vectors in $Z$, and can therefore simply focus on $\Q_\sigma$ vectors to easily discover $\sigma$. }
\end{remark}

We now further reduce the r.h.s. of \eqref{eq:nonadapt-likelihood} to events that are simpler to bound: the main strategy is to properly condition on intermediate events to reveal independences that induce simple i.i.d. Bernouilli's we can exploit. Towards this end, we first consider the high-probability events of the following proposition.

\begin{proposition} \label{prop:events} Let $Z\sim \Gamma_{-}$. 
Define ${\hat N}_P(Z), {\hat N}_\Q(Z)$ as the number of homogeneous vectors \\
$\braces{Z_t: \forall i,j \in [n], X_{t, i} = X_{t, j} = x_1 \, \land \,  Y_{t, i} = Y_{t, j}}$ generated respectively by $P_{-}$, and $Q_{-}$. 

Define the events (on $Z$): $$\bE_P \doteq \braces{\expec \brackets{{\hat N}_P}/2 \leq {\hat N}_P \leq 2 \expec \brackets{{\hat N}_P}} \text{ and } \bE_Q \doteq \braces{{\hat N}_\Q \leq 2 \expec\brackets {{\hat N}_\Q}}.$$
We have that $\prob\paren{\bE_P^\compl} \leq 2\exp\paren{-\expec \brackets{{\hat N}_P}/8}$ while 
$\prob\paren{\bE_Q^\compl} \leq \exp\paren{-\expec \brackets{{\hat N}_Q}/3}$.
\end{proposition} 

\begin{proof}
The proposition follows from multiplicative Chernoff bounds.
\end{proof}

Notice that, in the above proposition, for $Z\sim \Gamma_{-}$, the expectations in the events are given by  
\begin{align} 
\expec \brackets{{\hat N}_P(Z)} =  N_P P_X^n(x_1)\paren{\eta_{P, -}^n(x_1) + \eta_{P, +}^n(x_1)} 
\text{ and } \expec \brackets{{\hat N}_\Q(Z)} =  N_\Q. \label{eq:expecNPNQ}
\end{align}
By the proposition, we just need these quantities large enough for the events to hold with sufficient probability. 
The next proposition conditions on these events to reduce the likelihood to simpler events. 

\begin{proposition}[Further Reducing \eqref{eq:nonadapt-likelihood}]\label{prop:likelihoodReduction} Let $Z\sim \Gamma_{-}$. 
For $\sigma = \pm$, let ${\hat N}_{\sigma}(Z)$ and ${\hat n}_\sigma(Z)$ as defined in Proposition \ref{prop:likelihoodRatio}. Furthermore, let $\tilde n_\sigma (Z)\doteq {\hat n}_\sigma(Z) - (n\cdot {\hat N}_{\sigma}(Z))$ denote the total number of $\pm$ labels at $x_1$ excluding homogeneous vectors. 
 In all that follows we let $\prob$ denote $\prob_{\Gamma_{-}}$, and we drop the dependence on $Z$ for ease of notation. 
 
Let ${\hat N}_P, {\hat N}_\Q$ and the events $\bE_P, \bE_\Q$ as defined in Proposition \ref{prop:events}. Suppose that for some $\delta_1, \delta_2 >0$, we have 
\begin{enumerate} 
\item[(i)] $\prob\paren{{\hat N}_{+} > {\hat N}_{-} \mid  {\hat N}_P, {\hat N}_\Q}\ind{ \bE_P \cap \bE_\Q} \geq \delta_1\cdot\ind{ \bE_P \cap \bE_\Q}$.
\item[(ii)]  $\prob\paren{\tilde n_{+} > \tilde n_{-}} \geq \delta_2$, {further assuming that $n>1$ so that the event is well defined (i.e., $\tilde n_{\pm}$ are not both $0$)}. 
\end{enumerate}
 We then have that: 
\begin{align} 
\prob\paren{{\hat N}_{+} \geq {\hat N}_{-} \, \land \, {\hat n}_{+} \geq {\hat n}_{-}}
\geq  \delta_1\cdot \paren{\delta_2 - \prob\paren{\bE_P^\compl} - \prob\paren{\bE_\Q^\compl}}.
\end{align}
\end{proposition}
\begin{proof} 
Let $A \doteq \braces{{\hat n}_{+} \geq {\hat n}_{-}}$, $B\doteq \braces{{\hat N}_{+} > {\hat N}_{-}}$, and $\tilde A \doteq \braces{\tilde n_{+} > \tilde n_{-}}$. First notice that, by definition, $\tilde A \cap B \implies A \cap B$, so we just need to bound $\prob(\tilde A \cap B)$ below. We have: 
\begin{align} 
\prob\paren{\tilde A \cap B} &= \expec\brackets{\prob\paren{\tilde A\cap B \mid  {\hat N}_P, {\hat N}_Q}} \label{eq:redux1}
= {\expec\brackets{\prob\paren{\tilde A\mid {\hat N}_P, {\hat N}_Q} \cdot \prob\paren{B \mid {\hat N}_P, {\hat N}_Q}}}\\
&\geq \delta_1 \cdot \expec\brackets{\prob\paren{\tilde A\mid {\hat N}_P, {\hat N}_Q} \cdot \ind{\bE_P \cap \bE_Q}}\nonumber\\
& \geq \delta_1 \cdot \expec\brackets{\prob\paren{\tilde A\mid {\hat N}_P, {\hat N}_Q} - \ind{\bE_P^\compl \cup \bE_Q^\compl}}
= \delta_1\cdot \paren{\prob\paren{\tilde A} - \prob\paren{\bE_P^\compl\cup \bE_\Q^\compl}}\label{eq:redux2} \\
& \geq  \delta_1\cdot \paren{\delta_2 - \prob\paren{\bE_P^\compl} - \prob\paren{\bE_\Q^\compl}}. \nonumber 
\end{align}
In \eqref{eq:redux1}, we used the fact that, all vectors in $Z$ being independent, $\tilde A$ is independent of $B$ given any value of $\braces{{\hat N}_P, {\hat N}_Q}$; in \eqref{eq:redux2} we used the fact that for any $p \in [0, 1]$, we have $p\cdot \ind{\bE} = p - p\cdot \ind{\bE^\compl} \geq p - \ind{\bE^\compl}$. 
\end{proof}

The following is a known counterpart of Chernoff bounds, following from Slud's inequalities \cite{slud1977distribution} for Binomials.

\begin{lemma}[Anticoncentration (Slud's Inequality)] \label{lem:Slud's}
Let $\braces{X_i}_{i \in [m]}$ denote i.i.d. Bernouilli's with parameter $0< p\leq 1/2$. Then for any $0 \leq m_0 \leq m (1- 2p)$, we have 
\begin{align*} 
\prob \paren{\sum_{i\in [m]} X_i > mp + m_0} \geq \frac{1}{4}\exp \paren{\frac{-m_0^2}{m p (1-p)}}. 
\end{align*}
\end{lemma}
\begin{proof}
This form of the 
lower bound follows from 
Theorem 2.1 of \cite{slud1977distribution}
--
which here implies a lower 
bound $1-\Phi(m_0/\sqrt{mp(1-p)})$, 
for $\Phi$ the standard normal 
CDF -- 
together with the well-known bound: 
$1-\Phi(x) \geq \frac{1}{2} \left( 1 - \sqrt{1 - e^{-x^2}} \right) 
\geq \frac{1}{4} e^{-x^2}$ 
\cite{tate1953double}. 
See \cite{mousavi2010tight} for a \sk{detailed
derivation of a similar expression}.
\end{proof}

\begin{proposition}[$\delta_1$ from Proposition \ref{prop:likelihoodReduction}] 
\label{prop:likelihoodRedux1}
Pick any $\skr{0\leq \beta < 1}$, $1\leq n < 2/\beta - 1$, and $N_\Q \geq 1$. Let \skr{$0 < c_1 \leq 1/64$}, in the construction of $\eta_{P, \sigma}$, $\sigma = \pm$.
Suppose $N_P$ is sufficiently large so that 
\sk{$$(n \cdot N_P)^{(2-(n+1)\beta)/(2-\beta)} \geq 4\cdot N_\Q^2\cdot n\cdot 2^{4n}c_1^{-n}.$$}
Let $\bE = \bE_P \cap \bE_Q$ as defined in Proposition \ref{prop:events} over $Z\sim \Gamma_{-}$. Adopting the notation of Proposition \ref{prop:likelihoodReduction}, we have $$\prob\paren{{\hat N}_{+} > {\hat N}_{-} \mid  {\hat N}_P, {\hat N}_Q}\ind{\bE} \geq \frac{1}{12}\cdot\ind{\bE}.$$
\end{proposition}

\begin{proof} 
Consider $N_H = \braces{{\hat N}_P, {\hat N}_Q}$ such that $\bE \doteq \bE_P \cap \bE_Q$ holds. Let $\tilde B\doteq \braces {{\hat N}_{+} > \frac{1}{2} {\hat N}_P + \expec\brackets{{\hat N}_Q}}$, and notice that $\tilde B \subset \braces{{\hat N}_{+} > {\hat N}_{-}}$ under $\bE_Q \doteq \braces{\expec\brackets{{\hat N}_Q} \geq {\hat N}_Q/2}$. Therefore we only need to bound 
$\prob\paren{\tilde B \mid N_H} = \prob\paren{\tilde B \mid {\hat N}_P}$. Now, for $\sigma = \pm$, let ${\cal Z}_{P, \sigma}$ denote the set of homogeneous vectors in $\braces{Z_t: \forall i\in [n], X_{t, i} x_1 \, \land \,  Y_{t, i} = \sigma}$ generated by $P_{-}$, and 
notice that, conditioned on 
${\hat N}_P$, ${\hat N}_{+}$ is distributed as $\text{Binomial}({\hat N}_P, p)$, where 
\begin{align*}
    p &= \prob\paren{Z_t \in {\cal Z}_{P, +}\mid Z_t \in {\cal Z}_{P, +} \cup {\cal Z}_{P, -}, {\hat N}_P} = \prob\paren{Z_t \in {\cal Z}_{P, +}\mid Z_t \in {\cal Z}_{P, +} \cup {\cal Z}_{P, -}} 
    = \frac{\eta^n_{P, -}(x_1)}{\eta^n_{P, -}(x_1) + \eta^n_{P, +}(x_1)}. 
\end{align*}
Therefore, applying Lemma \ref{lem:Slud's}, 
\begin{align*} 
\prob\paren{{\hat N}_+ > {\hat N}_P \cdot p + \sqrt{{\hat N}_P \cdot p(1-p)} \mid {\hat N}_P} \geq \frac{1}{12}.
\end{align*}
We now just need to show that the event under the probability implies $\tilde B$, in other words, that 
\begin{align} 
\sqrt{{\hat N}_P \cdot p(1-p)} &\geq \frac{1}{2} {\hat N}_P - {\hat N}_P\cdot p + \expec\brackets{{\hat N}_\Q} 
 = {\hat N}_P\cdot \paren{\frac{1}{2} - p} + N_\Q. \label{eq:anticonc1}
\end{align}
Next we upper bound the r.h.s of \eqref{eq:anticonc1} and lower-bound its l.h.s. Under $\bE_P$ and using \eqref{eq:expecNPNQ}, we have: 
\begin{align} 
{\hat N}_P\cdot \paren{\frac{1}{2} - p} &\leq \expec\brackets{{\hat N}_P}\cdot \paren{(1-p) - p} 
= N_P P_X^n(x_1)\cdot \paren{\eta^n_{P, +}(x_1) - \eta^n_{P, -}(x_1)} \nonumber \\
&\leq N_P P_X^n(x_1)\cdot n \cdot \paren{\eta_{P, +}(x_1) - \eta_{P, -}(x_1)}, \label{eq:anticoncRhs}
\end{align}
where for the last inequality we used the fact that, for $a \geq b>0$, we have 
$a^n - b^n = (a-b) \sum_{k = 0}^{n-1} a^{n-1}\cdot \paren{\frac{b}{a}}^k$.
\sk{On the other hand, the conditions on $N_P$ let us lower-bound $\eta_{P, -}(x_1)$ by $1/4$}, and we have: 
\begin{align} 
{\hat N}_P \cdot p(1-p) &\geq \frac{1}{2}\expec \brackets{{\hat N}_P}\cdot p(1-p) 
= \frac{1}{2}N_P P_X^n(x_1)\cdot \frac{\eta^n_{P, -}(x_1)\cdot\eta^n_{P, +}(x_1)}{\eta^n_{P, -}(x_1) + \eta^n_{P, +}(x_1)} \nonumber \\
&\geq \skr{\frac{1}{2}N_P P_X^n(x_1)\cdot 2^{-3n}}. 
\label{eq:anticoncLhs}
\end{align}

Combining \eqref{eq:anticoncRhs} and \eqref{eq:anticoncLhs}, we see that $\tilde B$ is implied by
\begin{align*}
    2^{-(2n)}\sqrt{N_P P_X^n(x_1)} &\geq N_P P_X^n(x_1)\cdot n \cdot \paren{\eta_{P, +}(x_1) - \eta_{P, -}(x_1)} + N_\Q, \text{ which is in turn implied by}\\
    \sk{n^{-1}2^{-(4n)} c_1^n \paren{n\cdot N_P}^{\frac{2-(n+1)\beta}{2-\beta}}} &\geq \sk{2c_1^{2n}\paren{n\cdot  N_P}^{\frac{2-2n\beta}{2-\beta}} + 2N_\Q^2}. 
\end{align*}
The conditions of the proposition ensure that the above inequality holds. 
\end{proof}

We obtain a bound on $\prob (N_+ > {\hat N}_{-})$ as an immediate corollary to the above proposition. 
\begin{corollary} \label{cor:boundOnB}
Under the conditions of Proposition \ref{prop:likelihoodRedux1}, we have: 
$$\prob\paren{{\hat N}_{+} > {\hat N}_{-}} \geq \expec\brackets{\prob\paren{{\hat N}_{+} > {\hat N}_{-} \mid  {\hat N}_P, {\hat N}_Q}\ind{\bE}} 
\geq \frac{1}{12}\prob(\bE).$$
\end{corollary}

We now turn to the second condition of Proposition \ref{prop:likelihoodReduction}.

\begin{proposition}[$\delta_2$ from Proposition \ref{prop:likelihoodReduction}]  \label{prop:likelihoodRedux2}
Let \sk{$0\leq \beta < 1$}, $n>1$, and $N_\Q \geq 16$.  Let \sk{$0 < c_1 \leq 2^{-10}$} in the construction of $\eta_{P, \sigma}$, $\sigma = \pm$.
Suppose $N_P$ is sufficiently large so that
$$\text{(i)} \ N_P \geq 3 N_\Q \quad \text{and} \quad \text{(ii)} \  (n \cdot N_P)^{\frac{2-2\beta}{2-\beta}} 
\geq {4096\cdot n^2 \cdot c_1^{-1}}. $$

Let $Z \sim \Gamma_{-}$, and $\tilde n_\sigma = \tilde n_\sigma (Z), \, \sigma = \pm$ as defined in Proposition \ref{prop:likelihoodReduction}. We then have that 
$\prob\paren{\tilde n_{+} > \tilde n_{-}} \geq \frac{1}{48}$. 
\end{proposition} 
\begin{proof} 
Under the notation of Proposition \ref{prop:likelihoodReduction}, let ${\hat N}_{P, -} \doteq {\hat N}_P - {\hat N}_{+}$, and 
for homogeneity of notation herein, let ${\hat N}_{P, +} \doteq {\hat N}_{+}$. Fix $\delta > 0$ to be defined, and notice that if $\delta > n\cdot \paren{{\hat N}_{P, +} - {\hat N}_{P, -}}$, then the event  
$$ \tilde A_\delta \doteq \braces {\tilde n_{+} + n\cdot {\hat N}_{P, +} \geq \tilde n _{-} + n\cdot {\hat N}_{P, -} + \delta } 
\text{ implies }\braces{\tilde n_{+} > \tilde n_{-}}.$$ 
As a first step, we want to upper-bound $\paren{{\hat N}_{P, +} - {\hat N}_{P, -}}$. 
Let $V_\delta$ denote an upper-bound on the variance of this quantity:
we have by Bernstein's inequality that, 
for any $t \leq \sqrt{V_\delta}$, with probability at least $1- e^{-t^2/4}$, 
\begin{align}
\paren{{\hat N}_{P, +} - {\hat N}_{P, -}} \leq \expec \brackets {{\hat N}_{P, +} - {\hat N}_{P, -}} + t\sqrt{V_\delta} \leq t\sqrt{V_\delta}. \label{eq:DeltaBernstein}
\end{align}
We therefore set $\delta = 4n \cdot \sqrt{V_\delta}$, whereby, for $\sqrt{V_\delta} \geq 4$, the event of \eqref{eq:DeltaBernstein} \sk{(with $t=4$)} happens with probability at least $1- 1/48$. Hence, we set  
$ V_\delta \doteq 16 \lor \expec \brackets{{\hat N}_{P, +} + {\hat N}_{P, -}} \geq 16 \lor \text{Var}\paren{{\hat N}_{P, +} - {\hat N}_{P, -}}$, where $\expec \brackets{{\hat N}_{P, +} + {\hat N}_{P, -}} \doteq  \expec \brackets{{\hat N}_P}$ is given in equation \eqref{eq:expecNPNQ}. We now proceed with a lower-bound on $\prob(\tilde A_\delta)$. 

{Let ${\hat n}_{x_1}$ denote the number of points sampled from $P_{-}$ that fall on $x_1$. 
}
Notice that, conditioned on these samples' indices, $n_+ \doteq \tilde n_{+} + n\cdot {\hat N}_{P, +}$ is distributed as $\text{Binomial}\paren{{\hat n}_{x_1}, p}$, where $p = \eta_{P, -}(x_1)$, the probability of $+$ given that $x_1$ is sampled from $P_{-}$. Applying Lemma \ref{lem:Slud's}, and integrating over ${\hat N}_\Q$, we have 
\begin{align} 
\prob\paren{n_+ > {\hat n}_{x_1} \cdot p + \sqrt{{\hat n}_{x_1}\cdot p(1-p)}} \geq \frac{1}{12}. \label{eq:anticonc2}
\end{align}

Now notice that $\tilde A_\delta$ holds whenever ${\hat n}_{+} \geq \frac{1}{2} \paren{{\hat n}_{x_1} + \delta}$, since ${\hat n}_{x_1} = \paren{\tilde n_{+} + n\cdot {\hat N}_{P, +}) + (\tilde n _{-} + n\cdot {\hat N}_{P, -}} $. 
Under the event of \eqref{eq:anticonc2}, we have ${\hat n}_{+} > \frac{1}{2} \paren{{\hat n}_{x_1} + \delta}$, whenever it holds that 
\begin{align} 
\sqrt{{\hat n}_{x_1}\cdot p(1-p)} \geq \frac{1}{2} \paren{{\hat n}_{x_1} + \delta} - {\hat n}_{x_1}\cdot p 
= {\hat n}_{x_1} \paren{\frac{1}{2} - p} + \frac{1}{2}\delta. \label{eq:implyingtildeAdelta}
\end{align}
Next we bound ${\hat n}_{x_1}$ with high probability. Consider any value of ${\hat N}_\Q$ such that 
$\bE_\Q$ (from Proposition \ref{prop:events}) holds, i.e., ${\hat N}_\Q \leq 2N_\Q $. Conditioned on such ${\hat N}_\Q$, 
${\hat n}_{x_1}$ is itself a Binomial with 
$$\expec\brackets{{\hat n}_{x_1} \mid {\hat N}_\Q} =  n (N-{\hat N}_\Q) \cdot P_X(x_1) 
\geq \frac{1}{2} n N  \cdot P_X(x_1) \geq  \frac{1}{2} n \cdot N_P \cdot P_X(x_1) 
= \sk{\frac{1}{2} c_1\paren{n\cdot N_P}^{\frac{2-2\beta}{2 -\beta}}},$$
where for the first inequality we used the fact that $N_P \geq 3 N_\Q$. 
Hence, by a multiplicative Chernoff bound, 
\begin{align} 
\prob\left(\frac{1}{2}\expec\brackets{{\hat n}_{x_1}| {\hat N}_\Q} \leq {\hat n}_{x_1}\leq 2\expec\brackets{{\hat n}_{x_1} | {\hat N}_\Q} \middle| {\hat N}_Q\right)\cdot\ind{\bE_Q} \geq \paren{1- \frac{1}{48}}\ind{\bE_Q}, \label{eq:nx1bound}
\end{align} 
whenever $\skr{c_1}\sk{\paren{n\cdot N_P}^{\frac{2-2\beta}{2 -\beta}} \geq 40}$. Now, by Proposition \ref{prop:events}, $\bE_\Q$ holds with probability at least 
$1- 1/48$ whenever $ \expec \brackets{{\hat N}_\Q} = N_\Q > 16$. Thus, integrating \eqref{eq:nx1bound} over ${\hat N}_\Q$, we get that, with probability at least $1-1/24$, 
\begin{align} 
\frac{1}{4}n \cdot N_P \cdot P_X(x_1) \leq {\hat n}_{x_1} \leq 2 n \cdot N \cdot P_X(x_1) \leq {\frac{8}{3} n \cdot N_P \cdot P_X(x_1)}.
\label{eq:nx1bound2}
\end{align} 

Thus, bounding both sides of \eqref{eq:implyingtildeAdelta}, $\tilde A_\delta$ holds whenever a) 
the events of \eqref{eq:anticonc2} and \eqref{eq:nx1bound2} hold, and b) the following inequality is satisfied: 
\begin{align} 
\sqrt{\frac{1}{4} n\cdot N_P\cdot P_X(x_1) \cdot p(1-p)} &\geq {\frac{8}{3}}n\cdot N_P\cdot P_X(x_1) \paren{\frac{1}{2} - p} + 2n\cdot \sqrt{V_\delta}, \text{ which holds whenever } \nonumber \\
\sk{\frac{1}{4\sqrt{2}}c_1^{1/2}\paren{n\cdot N_P}^{\frac{1-\beta}{2-\beta}}} 
&\geq \sk{\frac{8}{3} c_1 \paren{n \cdot N_P}^{\frac{1-\beta}{2-\beta}} + 
2n\cdot\paren{4 \lor n^{-1/2}c_1^{n/2}\paren{n\cdot N_P}^{\frac{1-(n+1)\beta/2}{2-\beta}}} }, \label{eq:condforrho2}
\end{align} 
where the r.h.s. and l.h.s. of \eqref{eq:condforrho2} upper and lower bound, respectively, the r.h.s. and l.h.s. of the previous inequality (using the setting of $V_\delta$ and \eqref{eq:expecNPNQ}, and \skr{lower-bounding $p(1-p)$ by $1/8$ for $N_P$ as large as assumed}). By the conditions of the Proposition, \eqref{eq:condforrho2} is satisfied. Thus, $\tilde A_\delta$ holds with probability at least $1/24$ since the events of \eqref{eq:anticonc2} and \eqref{eq:nx1bound2} hold together with that probability (using the fact that $\prob(A \cap B) \geq \prob(A) - \prob(B^\compl)$). Finally, we can conclude that $\braces{\tilde n_{+} > \tilde n_{-}}$ with probability at least $1/48$ since $\tilde A_\delta$ and the event of \eqref{eq:DeltaBernstein} hold together with that probability. 
\end{proof}

Finally we bound the $\Qt_\pm$ term in the likelihood equation \eqref{eq:nonadapt-likelihood}. 

\begin{proposition}[$\Qt_{+}/\Qt_{-}$] \label{prop:targetlikelihood} Let $n_\Qt >0$. Again let $Z \sim \Gamma_{-}$, and let $0< c_0 \leq 1/4$ in the construction of 
$\eta_{\Qt, \sigma}, \sigma = \pm$. 
$$\prob\paren{\frac{\Qt_{+}(Z_{N+1})}{\Qt_{-}(Z_{N+1})} \geq  1} \geq \frac{1}{84}.$$
\end{proposition} 

The proof is given in Appendix \ref{app:nonadapt}, and follows similar lines as above, namely, isolate sufficient statistics (number of $\pm$ in $Z_{N+1}$) and concluding by anticoncentration upon proper conditioning. 

We can now combine all the above analysis into the following main proposition. 

\begin{proposition}\label{prop:noAdapt}
Pick any $0\leq \beta < 1$, $1\leq n < 2/\beta - 1$, and $N_\Q \geq 16$. Let \sk{$0 < c_1 \leq 2^{-10}$}, and $0< c_0 \leq  1/4$ in the constructions of $\eta_{P, \sigma}$, $\eta_{\Qt, \sigma}$, $\sigma = \pm$. 
Suppose $N_P$ is sufficiently large so that $N_P\geq 3N_\Q$, and also 
\sk{
\begin{itemize} 
\item[(i)] $(n \cdot N_P)^{\frac{2-2\beta}{2-\beta}} 
\geq {4096\cdot n^2 \cdot c_1^{-1}}$.
\item[(ii)] $(n \cdot N_P)^{(2-(n+1)\beta)/(2-\beta)} \geq 4\cdot N_\Q^2\cdot n\cdot 2^{4n}c_1^{-n}$.
\item [(iii)] $N_P^{(2- (n+1)\beta)/(2-\beta)}\geq 22\cdot n\cdot 2^n\cdot \skr{c_1^{-n}}$.
\end{itemize}
}
Let $\hat h$ denote any classification procedure having access to $Z\sim \Gamma_\sigma, \sigma = \pm$. We then have that 
\begin{align*}
    \inf_{\hat h} \sup_{\sigma \in \braces{\pm}} \prob_{\Gamma_\sigma}\paren{\E_{\Qt}(\hat h) \geq c_0\cdot  \paren{1 \land n_\Qt^{-1/(2-\beta)}}} \geq \frac{1}{12}\cdot\frac{1}{96}\cdot \frac{1}{84}.
\end{align*}

\end{proposition}

\begin{proof}
Following the above propositions, again assume w.l.o.g. that $Z\sim \Gamma_{-}$. 
Let $\bE_P, \bE_Q$ as defined in Proposition \ref{prop:events} over $Z\sim \Gamma_{-}$, and notice that, under our assumptions on $N_\Q$ and \sk{\it (iii)} on $N_P$, each of these events occurs with probability at least $1-1/(2\cdot 96)$. 

Thus, for $n>1$, by Propositions \ref{prop:likelihoodRatio}, \ref{prop:likelihoodReduction}, and \ref{prop:targetlikelihood}, we have that 
$\prob_{\Gamma_{-}}\paren{\Gamma_{+}(Z)> \Gamma_{-}(Z)}$ is at least 
$\delta_1(\delta_2 - 1/96)\frac{1}{84}$. Now plug in $\delta_1 = 1/12$ and $\delta_2 = 1/48$ from Propositions 
\ref{prop:likelihoodRedux1} and \ref{prop:likelihoodRedux2}. For $n =1$, using Proposition \ref{prop:likelihoodRatio} and \ref{prop:targetlikelihood}, and noticing that $\{{\hat N}_{+}(Z) > {\hat N}_{-}(Z)\} \implies \{{\hat n}_{+}(Z) \geq {\hat n}_{-}(Z)\}$, we can conclude by Corollary \ref{cor:boundOnB} that $\prob_{\Gamma_{-}}\paren{\Gamma_{+}(Z)> \Gamma_{-}(Z)}$ is at least 
$\frac{1}{12}\prob\paren{\bE_P\cap\bE_Q}\cdot\frac{1}{84}$, again matching the lower-bound in the statement.  

Now, if $\hat h$ wrongly picks $\sigma = +$ (i.e. picks $h \in \H$, $h(x_1) = +$), then $\E_\Qt(\hat h) \geq \Qt_X(x_1)\cdot \paren{\eta_{\Qt, +} - \eta_{\Qt, -}} =   c_0\epsilon_0$. By Remark \ref{rem:likelihood}, for any $\hat h$, the probability that $\hat h$ picks $\sigma = +$ is bounded below by $\prob_{\Gamma_{-}}\paren{\Gamma_{+}(Z)> \Gamma_{-}(Z)}$.  
\end{proof}

We can now conclude with the proof of the main result of this section. 

\begin{proof}[Proof of Theorem \ref{thm:nonadaptivityfull}] 
The first part of the statement builds on Proposition \ref{prop:noAdapt} as follows. 
\skr{Set $c_0 = 1/4$ and $c_1 = 2^{-10}$}.
First, let $Z \sim \Gamma_\sigma$, and let ${\hat N}_\Q$ denote the number of vectors $Z_t \in Z$ that were generated by $\Q$ (as in Proposition \ref{prop:events}). Let 
$\bE_{\Q, \sharp}$ denote the event that ${\hat N}_Q \geq \expec [{\hat N}_\Q]/2 = N_\Q/2$, 
Let $\bE_\E \doteq \braces{\E_{\Qt}(\hat h) \geq c_0\cdot  \epsilon_{0}}$. By Proposition \ref{prop:noAdapt}, for some $\sigma \in \braces{\pm}$, we have that $\prob\paren{\bE_\E}$ is bounded below.

Now decouple the randomness in $Z$ as follows. Let $\zeta \doteq \braces{\zeta_t}_{t\in [N]}$ denote $N$ i.i.d. choices of $P_\sigma $ or $\Q_\sigma $ with respective probabilities $\alpha_P = N_P/N$ and $\alpha_\Q = N_\Q/N$; choose $Z_t\in Z$ according to $\zeta_t^n$. We then have that 
\begin{align*}
  \expec \brackets{\prob\paren{\bE_\E \mid \zeta} \mid  \bE_{\Q, \sharp}} 
  &\geq 
 \expec \brackets{\prob\paren{\bE_\E \mid \zeta}\cdot \ind{\bE_{\Q, \sharp}}} \\
 &\geq \expec \brackets{\prob\paren{\bE_\E \mid \zeta}} - \prob\paren{\bE_{\Q, \sharp}^\compl} 
 = \prob\paren{\bE_\E} - \prob\paren{\bE_{\Q, \sharp}^\compl} \geq c, 
\end{align*}
where we can bound $\prob\paren{\bE_{\Q, \sharp}^\compl}$ how ever small for $N_\Q$ sufficiently large (by a multiplicative Chernoff). Now conclude by noticing that the {above conditional expectation is a projection of the measure $\alpha^N$ onto $\M$ (via the injection $\zeta \mapsto \Pi \in \M$)} and is bounded below, implying $\sup_{\zeta \mid \bE_{\Q, \sharp}} \prob\paren{\bE_\E \mid \zeta}$ must be bounded below.

The second part of the theorem is a direct consequence of the results of Section \ref{sec:semiadaptive}. 
\end{proof}

\bibliographystyle{alpha}
\bibliography{refs}

\newcommand{\etalchar}[1]{$^{#1}$}
\begin{thebibliography}{LPVDG{\etalchar{+}}11}

\bibitem[ABGLP19]{arjovsky2019invariant}
Martin Arjovsky, L{\'e}on Bottou, Ishaan Gulrajani, and David Lopez-Paz.
\newblock Invariant risk minimization.
\newblock {\em arXiv:1907.02893}, 2019.

\bibitem[AKK{\etalchar{+}}19]{arora2019theoretical}
Sanjeev Arora, Hrishikesh Khandeparkar, Mikhail Khodak, Orestis Plevrakis, and
  Nikunj Saunshi.
\newblock A theoretical analysis of contrastive unsupervised representation
  learning.
\newblock {\em arXiv:1902.09229}, 2019.

\bibitem[APMS19]{achille2019information}
Alessandro Achille, Giovanni Paolini, Glen Mbeng, and Stefano Soatto.
\newblock The information complexity of learning tasks, their structure and
  their distance.
\newblock {\em arXiv:1904.03292}, 2019.

\bibitem[AZ05]{ando2005framework}
Rie~Kubota Ando and Tong Zhang.
\newblock A framework for learning predictive structures from multiple tasks
  and unlabeled data.
\newblock {\em Journal of Machine Learning Research}, 6(Nov):1817--1853, 2005.

\bibitem[Bar92]{bartlett1992learning}
Peter~L Bartlett.
\newblock Learning with a slowly changing distribution.
\newblock In {\em Proceedings of the 5th Annual Workshop on Computational
  Learning Theory}, 1992.

\bibitem[Bax97]{Baxter-Bayesian}
Jonathan Baxter.
\newblock A {B}ayesian/information theoretic model of learning to learn via
  multiple task sampling.
\newblock {\em Machine Learning}, 28(1):7--39, 1997.

\bibitem[BDB08]{ben2008notion}
Shai Ben-David and Reba~Schuller Borbely.
\newblock A notion of task relatedness yielding provable multiple-task learning
  guarantees.
\newblock {\em Machine Learning}, 73(3):273--287, 2008.

\bibitem[BDBC{\etalchar{+}}10]{ben2010theory}
Shai Ben-David, John Blitzer, Koby Crammer, Alex Kulesza, Fernando Pereira, and
  Jennifer~Wortman Vaughan.
\newblock A theory of learning from different domains.
\newblock {\em Machine Learning}, 79(1-2):151--175, 2010.

\bibitem[BDBCP07]{ben2007analysis}
Shai Ben-David, John Blitzer, Koby Crammer, and Fernando Pereira.
\newblock Analysis of representations for domain adaptation.
\newblock In {\em Advances in Neural Information Processing Systems}, 2007.

\bibitem[BDBM89]{ben1989parametrization}
Shai Ben-David, Gyora~M Benedek, and Yishay Mansour.
\newblock A parametrization scheme for classifying models of learnability.
\newblock In {\em Proceedings of the 2nd Annual Workshop on Computational
  Learning Theory}, 1989.

\bibitem[BDLLP10]{david2010impossibility}
Shai Ben-David, Tyler Lu, Teresa Luu, and D{\'a}vid P{\'a}l.
\newblock Impossibility theorems for domain adaptation.
\newblock In {\em Proceedings of the 13th International Conference on
  Artificial Intelligence and Statistics}, 2010.

\bibitem[BHPQ17]{blum2017collaborative}
Avrim Blum, Nika Haghtalab, Ariel~D Procaccia, and Mingda Qiao.
\newblock Collaborative {PAC} learning.
\newblock In {\em Advances in Neural Information Processing Systems}, 2017.

\bibitem[BL97]{barve1997complexity}
Rakesh~D Barve and Philip~M Long.
\newblock On the complexity of learning from drifting distributions.
\newblock {\em Information and Computation}, 138(2):170--193, 1997.

\bibitem[BLM13]{boucheron2013concentration}
St{\'e}phane Boucheron, G{\'a}bor Lugosi, and Pascal Massart.
\newblock {\em Concentration Inequalities: A Nonasymptotic Theory of
  Independence}.
\newblock Oxford university press, 2013.

\bibitem[Car97]{caruana1997multitask}
Rich Caruana.
\newblock Multitask learning.
\newblock {\em Machine Learning}, 28(1):41--75, 1997.

\bibitem[CKW08]{crammer2008learning}
Koby Crammer, Michael Kearns, and Jennifer Wortman.
\newblock Learning from multiple sources.
\newblock {\em Journal of Machine Learning Research}, 9(Aug):1757--1774, 2008.

\bibitem[CMRR08]{cortes2008sample}
Corinna Cortes, Mehryar Mohri, Michael Riley, and Afshin Rostamizadeh.
\newblock Sample selection bias correction theory.
\newblock In {\em International Conference on Algorithmic Learning Theory},
  2008.

\bibitem[DHK{\etalchar{+}}20]{du2020few}
Simon~S Du, Wei Hu, Sham~M Kakade, Jason~D Lee, and Qi~Lei.
\newblock Few-shot learning via learning the representation, provably.
\newblock {\em arXiv:2002.09434}, 2020.

\bibitem[GSH{\etalchar{+}}09]{gretton2009covariate}
Arthur Gretton, Alex Smola, Jiayuan Huang, Marcel Schmittfull, Karsten
  Borgwardt, and Bernhard Sch{\"o}lkopf.
\newblock Covariate shift by kernel mean matching.
\newblock In {\em Dataset Shift in Machine Learning}, pages 131--160, 2009.

\bibitem[HK19]{hanneke2019value}
Steve Hanneke and Samory Kpotufe.
\newblock On the value of target data in transfer learning.
\newblock In {\em Advances in Neural Information Processing Systems}, 2019.

\bibitem[HY19]{hanneke19a}
Steve Hanneke and Liu Yang.
\newblock Statistical learning under nonstationary mixing processes.
\newblock In {\em Proceedings of the 22nd International Conference on
  Artificial Intelligence and Statistics}, 2019.

\bibitem[JSRR10]{jalali2010dirty}
Ali Jalali, Sujay Sanghavi, Chao Ruan, and Pradeep Ravikumar.
\newblock A dirty model for multi-task learning.
\newblock In {\em Advances in Neural Information Processing Systems}, 2010.

\bibitem[KFAL20]{konstantinov2020sample}
Nikola Konstantinov, Elias Frantar, Dan Alistarh, and Christoph~H Lampert.
\newblock On the sample complexity of adversarial multi-source {PAC} learning.
\newblock {\em arXiv:2002.10384}, 2020.

\bibitem[KM18]{kpotufe2018marginal}
Samory Kpotufe and Guillaume Martinet.
\newblock Marginal singularity, and the benefits of labels in covariate-shift.
\newblock {\em arXiv:1803.01833}, 2018.

\bibitem[Kol06]{koltchinskii:06}
V.~Koltchinskii.
\newblock Local {R}ademacher complexities and oracle inequalities in risk
  minimization.
\newblock {\em The Annals of Statistics}, 34(6):2593--2656, 2006.

\bibitem[Kol11]{koltchinskii2011oracle}
Vladimir Koltchinskii.
\newblock {\em Oracle Inequalities in Empirical Risk Minimization and Sparse
  Recovery Problems: Ecole d’Et{\'e} de Probabilit{\'e}s de Saint-Flour
  XXXVIII-2008}, volume 2033.
\newblock 2011.

\bibitem[KR05]{klein2005concentration}
Thierry Klein and Emmanuel Rio.
\newblock Concentration around the mean for maxima of empirical processes.
\newblock {\em The Annals of Probability}, 33(3):1060--1077, 2005.

\bibitem[LPVDG{\etalchar{+}}11]{lounici2011oracle}
Karim Lounici, Massimiliano Pontil, Sara Van De~Geer, Alexandre~B Tsybakov,
  et~al.
\newblock Oracle inequalities and optimal inference under group sparsity.
\newblock {\em The Annals of Statistics}, 39(4):2164--2204, 2011.

\bibitem[MB17]{mcnamara2017risk}
Daniel McNamara and Maria-Florina Balcan.
\newblock Risk bounds for transferring representations with and without
  fine-tuning.
\newblock In {\em International Conference on Machine Learning}, 2017.

\bibitem[MBS13]{muandet2013domain}
Krikamol Muandet, David Balduzzi, and Bernhard Sch{\"o}lkopf.
\newblock Domain generalization via invariant feature representation.
\newblock In {\em International Conference on Machine Learning}, 2013.

\bibitem[MM12]{mohri2012new}
Mehryar Mohri and Andres~Munoz Medina.
\newblock New analysis and algorithm for learning with drifting distributions.
\newblock In {\em International Conference on Algorithmic Learning Theory},
  2012.

\bibitem[MMM19]{mahloujifar2019universal}
Saeed Mahloujifar, Mohammad Mahmoody, and Ameer Mohammed.
\newblock Universal multi-party poisoning attacks.
\newblock In {\em Proceedings of the 36th International Conference on Machine
  Learning}, 2019.

\bibitem[MMR09]{mansour2009multiple}
Yishay Mansour, Mehryar Mohri, and Afshin Rostamizadeh.
\newblock Multiple source adaptation and the {R}{\'e}nyi divergence.
\newblock In {\em Proceedings of the 25th Conference on Uncertainty in
  Artificial Intelligence}, 2009.

\bibitem[Mou10]{mousavi2010tight}
Nima Mousavi.
\newblock How tight is {C}hernoff bound?
\newblock
  \url{https://ece.uwaterloo.ca/~nmousavi/Papers/Chernoff-Tightness.pdf}, 2010.

\bibitem[MPRP13]{maurer2013sparse}
Andreas Maurer, Massi Pontil, and Bernardino Romera-Paredes.
\newblock Sparse coding for multitask and transfer learning.
\newblock In {\em International Conference on Machine Learning}, 2013.

\bibitem[MPRP16]{maurer2016benefit}
Andreas Maurer, Massimiliano Pontil, and Bernardino Romera-Paredes.
\newblock The benefit of multitask representation learning.
\newblock {\em The Journal of Machine Learning Research}, 17(1):2853--2884,
  2016.

\bibitem[NW11]{negahban5773043}
S.~N. {Negahban} and M.~J. {Wainwright}.
\newblock Simultaneous support recovery in high dimensions: Benefits and perils
  of block $\ell _{1}/\ell _{\infty} $-regularization.
\newblock {\em IEEE Transactions on Information Theory}, 57(6):3841--3863,
  2011.

\bibitem[PL14]{pentina2014pac}
Anastasia Pentina and Christoph Lampert.
\newblock A {PAC}-{B}ayesian bound for lifelong learning.
\newblock In {\em International Conference on Machine Learning}, 2014.

\bibitem[PM13]{pontil2013excess}
Massimiliano Pontil and Andreas Maurer.
\newblock Excess risk bounds for multitask learning with trace norm
  regularization.
\newblock In {\em Conference on Learning Theory}, 2013.

\bibitem[Qia18]{qiao2018outliers}
Mingda Qiao.
\newblock Do outliers ruin collaboration?
\newblock {\em arXiv:1805.04720}, 2018.

\bibitem[RHS17]{redko2017theoretical}
Ievgen Redko, Amaury Habrard, and Marc Sebban.
\newblock Theoretical analysis of domain adaptation with optimal transport.
\newblock In {\em Joint European Conference on Machine Learning and Knowledge
  Discovery in Databases}, 2017.

\bibitem[Sau72]{sauer1972density}
Norbert Sauer.
\newblock On the density of families of sets.
\newblock {\em Journal of Combinatorial Theory, Series A}, 13(1):145--147,
  1972.

\bibitem[Slu77]{slud1977distribution}
Eric~V Slud.
\newblock Distribution inequalities for the binomial law.
\newblock {\em The Annals of Probability}, pages 404--412, 1977.

\bibitem[SQZY18]{shen2018wasserstein}
Jian Shen, Yanru Qu, Weinan Zhang, and Yong Yu.
\newblock Wasserstein distance guided representation learning for domain
  adaptation.
\newblock In {\em 32nd {AAAI} Conference on Artificial Intelligence}, 2018.

\bibitem[SZ19]{scott2019learning}
Clayton Scott and Jianxin Zhang.
\newblock Learning from multiple corrupted sources, with application to
  learning from label proportions.
\newblock {\em arXiv:1910.04665}, 2019.

\bibitem[Tat53]{tate1953double}
Robert~F Tate.
\newblock On a double inequality of the normal distribution.
\newblock {\em The Annals of Mathematical Statistics}, 24(1):132--134, 1953.

\bibitem[TJJ20]{tripuraneni2020theory}
Nilesh Tripuraneni, Michael~I Jordan, and Chi Jin.
\newblock On the theory of transfer learning: {T}he importance of task
  diversity.
\newblock {\em arXiv:2006.11650}, 2020.

\bibitem[Tsy04]{tsybakov2004optimal}
Alexander~B. Tsybakov.
\newblock Optimal aggregation of classifiers in statistical learning.
\newblock {\em The Annals of Statistics}, 32(1):135--166, 2004.

\bibitem[Tsy09]{tsybakov2009introduction}
Alexandre~B Tsybakov.
\newblock {\em Introduction to Nonparametric Estimation}.
\newblock Springer, 2009.

\bibitem[VC71]{VC:72}
V.~Vapnik and A.~Chervonenkis.
\newblock On the uniform convergence of relative frequencies of events to their
  expectation.
\newblock {\em Theory of Probability and its Applications}, 16:264--280, 1971.

\bibitem[vW96]{van-der-Vaart:96}
A.~W. {van der Vaart} and J.~A. Wellner.
\newblock {\em Weak Convergence and Empirical Processes}.
\newblock Springer-Verlag New York, 1996.

\bibitem[YHC13]{yang2013theory}
Liu Yang, Steve Hanneke, and Jaime Carbonell.
\newblock A theory of transfer learning with applications to active learning.
\newblock {\em Machine learning}, 90(2):161--189, 2013.

\bibitem[ZL19]{zimintasks}
Alexander Zimin and Christoph~H Lampert.
\newblock Tasks without borders: A new approach to online multi-task learning.
\newblock In {\em Workshop on Adaptive \& Multitask Learning}, 2019.

\end{thebibliography}

\newpage 
\appendix 

{\bf \large Appendix}

\section{Proof of Lemma~\ref{lem:uniform-bernstein}}
\label{app:uniform-bernstein}

Let 
$W_{1:m} = (W_1,\ldots,W_{m})$ be a vector of 
independent $\W$-valued random variables (for some space $\W$),
not necessarily identically distributed.
Let $\F$ be a set of measurable functions $\W \to [-1,1]$, 
and let $|\F(m)|$ be the number of distinct vectors possible on 
$m$ points.
Let $\alpha = \{\alpha_1,\ldots,\alpha_m\} \in [0,1]^m$.
Define 
$\hat{\meas}_{\alpha}(f) \doteq \sum_{i=1}^{m} \alpha_i f(W_i)$,
and 
$\meas_{\alpha}(f) \doteq \EE\!\left[ \hat{\meas}_{\alpha}(f) \right]$, 
and also $\hat{\sigma}_{\alpha}^2(f) \doteq \sum_{i=1}^{m} \alpha_i^2 f^2(W_i)$,
and $\sigma_{\alpha}^2(f) \doteq \EE \hat{\sigma}_{\alpha}^2(f)$.
Define 
$\F_{\sigma} \doteq \{ f \in \F : \sigma_{\alpha}(f) \leq \sigma \}$.

The following lemma is immediate from a result of \cite{klein2005concentration} 
(see also \cite{boucheron2013concentration} Section 12.5).

\begin{lemma}
\label{lem:non-id-bousquet}
For any $\sigma > 0$ with $\F_{\sigma} \neq \emptyset$, 
letting $L_{\sigma} \doteq \sup_{f \in \F_{\sigma}} \left( \hat{\meas}_{\alpha}(f) - \meas_{\alpha}(f) \right)$, $\forall \epsilon > 0$, 
\begin{align*}
\P\!\left( L_{\sigma} \geq \EE\!\left[ L_{\sigma} \right] + 2\epsilon \right)
\leq \exp\!\left\{ - \frac{\epsilon}{4} \ln\!\left( 1 + 2\ln\!\left( 1 + \frac{\epsilon}{2\EE\!\left[ L_{\sigma} \right] + \sigma^2} \right)\right)  \right\}.
\end{align*}
\end{lemma}

In particular, based on the inequality $\ln(1+x) \geq \frac{x}{1+x}$, 
this implies 
\begin{equation*}
\P\!\left( L_{\sigma} \geq \EE\!\left[ L_{\sigma} \right] + 2\epsilon \right)
\leq \exp\!\left\{ - \frac{\epsilon^2}{6\epsilon + 4\EE\!\left[ L_{\sigma} \right] + 2\sigma^2}  \right\}.
\end{equation*}
In particular, for any $\delta > 0$, setting 
\begin{equation*}
\epsilon = 12 \max\!\left\{ \sqrt{ (\EE[L_{\sigma}]+\sigma^2) \ln\!\left(\frac{1}{\delta}\right)}, \ln\!\left(\frac{1}{\delta}\right) \right\} 
\end{equation*}
reveals that, with probability at least $1-\delta$, 
\begin{equation}
\label{eqn:non-id-bousquet}
L_{\sigma} \leq \EE[L_{\sigma}] + 24 \max\!\left\{ \sqrt{ (\EE[L_{\sigma}]+\sigma^2) \ln\!\left(\frac{1}{\delta}\right)}, \ln\!\left(\frac{1}{\delta}\right) \right\}.
\end{equation}

Next, the following lemma bounds $\EE[L_{\sigma}]$ using a standard route. 

\begin{lemma}
\label{lem:expected-bernstein}
There is a numerical constant $C \geq 1$ such that,
for any $\sigma > 0$ with $\F_{\sigma} \neq \emptyset$, 
for $L_{\sigma}$ as in Lemma~\ref{lem:non-id-bousquet},
\begin{equation*}
\EE[L_{\sigma}] 
\leq C \sqrt{ \sigma^2 \ln(|\F(m)|) } + C \ln(|\F(m)|).
\end{equation*}
\end{lemma}
\begin{proof}
From Lemma 11.4 of \cite{boucheron2013concentration}, 
for $\epsilon_1,\ldots,\epsilon_m$ independent ${\rm Uniform}(\{-1,1\})$ 
(and independent of $W_{1:m}$), 
\begin{equation*}
\EE[L_{\sigma}] \leq 2 \EE\!\left[ \sup_{f \in \F_{\sigma}} \sum_{i=1}^{m} \epsilon_i \alpha_i \left(f(W_i) - \EE[ f(W_i) ]\right) \right].
\end{equation*}
Furthermore, from \cite{boucheron2013concentration} (Exercise 13.2, 
based on a result proved in \cite{van-der-Vaart:96}; 
see also \cite{koltchinskii2011oracle}), 
there is a numerical constant 
$C > 0$ such that 
\begin{align*}
\EE\!\left[ \sup_{f \in \F_{\sigma}} \sum_{i=1}^{m} \epsilon_i \alpha_i \left(f(W_i) - \EE[ f(W_i) ]\right) \right]
\leq C \sigma \sqrt{ \log(|\F(m)|) }
+ C \log(|\F(m)|).  
\end{align*}
\end{proof}

In particular, the above results imply the following lemma.

\begin{lemma}
\label{lem:unlocalized-bernstein}
There exists a numerical constant $C \geq 1$ such that, 
for any $\delta \in (0,1)$, with probability at least $1-\delta$, 
every $f \in \F$ satisfies
\begin{equation*}
\hat{\meas}_{\alpha}(f) \leq \meas_{\alpha}(f)
+ C \sqrt{\sigma_{\alpha}^2(f) \ln\!\left(\frac{|\F(m)|\log_{2}(2m)}{\delta}\right)} 
+ C \ln\!\left(\frac{|\F(m)|\log_{2}(2m)}{\delta}\right).
\end{equation*}
\end{lemma}
\begin{proof}
Let $\delta_m \doteq \frac{\delta}{\log_{2}(2m)}$.
Combining Lemma~\ref{lem:expected-bernstein} with \eqref{eqn:non-id-bousquet} 
implies that, letting $\sigma_{k}^2 = 2^{k}$ ($k \in \{0,\ldots,\log_{2}(m)\}$),
with probability at least $1 - \delta$ (by a union bound), 
every $k \in \{0,\ldots,\log_{2}(m)\}$ has (for a numerical constant $C \geq 1$)
\begin{align*}
L_{\sigma_{k}} & \leq 24 \max\!\left\{ \sqrt{ (\EE[L_{\sigma_{k}}]+\sigma_{k}^2) \ln\!\left(\frac{1}{\delta_m}\right)}, \ln\!\left(\frac{1}{\delta_{m}}\right) \right\}
+ C \sqrt{ \sigma_{k}^2 \ln(|\F(m)|) } + C \ln(|\F(m)|)
\\ & \leq 24 \sqrt{ (\EE[L_{\sigma_{k}}]+\sigma_{k}^2) \ln\!\left(\frac{1}{\delta_{m}}\right)} + C \sqrt{ \sigma_{k}^2 \ln(|\F(m)|) } + (24+C) \ln\!\left(\frac{|\F(m)|}{\delta_{m}}\right)
\\ & \leq 24 \sqrt{2 \sigma_{k}^2 \ln\!\left(\frac{1}{\delta_{m}}\right)} 
+ 48 \sqrt{ C \sqrt{ \sigma_{k}^2 \ln(|\F(m)|) } \ln\!\left(\frac{1}{\delta_{m}}\right) }
+ 48 \sqrt{ C \ln(|\F(m)|) \ln\!\left(\frac{1}{\delta_{m}}\right) }
\\ & \phantom{\leq} + C \sqrt{ \sigma_{k}^2 \ln(|\F(m)|) } + (24+C) \ln\!\left(\frac{|\F(m)|}{\delta_{m}}\right)
\\ & \leq (72+2C+48\sqrt{C}) \left(  \sqrt{\sigma_{k}^2 \ln\!\left(\frac{|\F(m)|}{\delta_{m}}\right)} 
+ \ln\!\left(\frac{|\F(m)|}{\delta_{m}}\right) \right).
\end{align*}
Suppose this event occurs.
In particular, for each $f \in \F$, let $k(f) \doteq \min\{ k : f \in \F_{\sigma_k} \}$.
Then every $f \in \F$ has (for a universal numerical constant $C^{\prime} \geq 1$) 
\begin{align*}
\hat{\meas}_{\alpha}(f) 
& \leq \meas_{\alpha}(f)
+ C^{\prime} \sqrt{\sigma_{k(f)}^2 \ln\!\left(\frac{|\F(m)|}{\delta_{m}}\right)} 
+ C^{\prime} \ln\!\left(\frac{|\F(m)|}{\delta_{m}}\right)
\\ & \leq \meas_{\alpha}(f)
+ C^{\prime} \sqrt{\max\!\left\{ 1, 2\sigma_{\alpha}^2(f) \right\} \ln\!\left(\frac{|\F(m)|}{\delta_{m}}\right)} 
+ C^{\prime} \ln\!\left(\frac{|\F(m)|}{\delta_{m}}\right)
\\ & \leq \meas_{\alpha}(f)
+ C^{\prime} \sqrt{2\sigma_{\alpha}^2(f) \ln\!\left(\frac{|\F(m)|}{\delta_{m}}\right)} 
+ 2C^{\prime} \ln\!\left(\frac{|\F(m)|}{\delta_{m}}\right).
\end{align*}
\end{proof}

We can also state a concentration result specific to the 
$\hat{\sigma}_{\alpha}^{2}(f)$ values, as follows.

\begin{lemma}
\label{lem:bernstein-variance-concentration}
There exists a numerical constant $C \geq 1$ such that, 
for any $\delta \in (0,1)$, with probability at least $1-\delta$, 
every $f \in \F$ satisfies
\begin{equation*}
\frac{1}{2} \sigma_{\alpha}^{2}(f) - C \ln\!\left(\frac{|\F(m)|\log_{2}(2m)}{\delta}\right)
\leq \hat{\sigma}_{\alpha}^{2}(f)
\leq 2 \sigma_{\alpha}^{2}(f) + C \ln\!\left(\frac{|\F(m)|\log_{2}(2m)}{\delta}\right).
\end{equation*}
\end{lemma}
\begin{proof}
Define $\F' \doteq \{ f^2 : f \in \F \}$ 
and note that $|\F'(m)| \leq |\F(m)|$.
Also define $\alpha'_i \doteq \alpha_i^2$.
Applying Lemma~\ref{lem:unlocalized-bernstein} with this $\F'$ and $\alpha'$, 
we have that, with probability at least $1-\delta/2$, 
every $f \in \F$ satisfies (for some numerical constant $C \geq 1$)
\begin{align}
\hat{\sigma}_{\alpha}^{2}(f) 
& \leq \sigma_{\alpha}^{2}(f)
+ C \sqrt{ \EE\!\left[\sum_{i=1}^{m} \alpha_i^4 f^4(W_i)\right] \ln\!\left(\frac{|\F(m)|2\log_{2}(2m)}{\delta}\right)} 
+ C \ln\!\left(\frac{|\F(m)|2\log_{2}(2m)}{\delta}\right)
\notag \\ & \leq \sigma_{\alpha}^{2}(f)
+ C \sqrt{ \sigma_{\alpha}^{2}(f) \ln\!\left(\frac{|\F(m)|2\log_{2}(2m)}{\delta}\right)} 
+ C \ln\!\left(\frac{|\F(m)|2\log_{2}(2m)}{\delta}\right). \label{eqn:bvb-1}
\end{align}
If $\sigma_{\alpha}^{2}(f) \geq C^2 \ln\!\left(\frac{|\F(m)|2\log_{2}(2m)}{\delta}\right)$ 
then the expression in \eqref{eqn:bvb-1} is at most 
\begin{equation*}
2 \sigma_{\alpha}^{2}(f)
+ C \ln\!\left(\frac{|\F(m)|2\log_{2}(2m)}{\delta}\right),
\end{equation*}
while if $\sigma_{\alpha}^{2}(f) \leq C^2 \ln\!\left(\frac{|\F(m)|2\log_{2}(2m)}{\delta}\right)$ 
then the expression in \eqref{eqn:bvb-1} is at most 
\begin{equation*}
\sigma_{\alpha}^{2}(f)
+ (C^2+C) \ln\!\left(\frac{|\F(m)|2\log_{2}(2m)}{\delta}\right).
\end{equation*}
Thus, either way we have 
\begin{equation*}
\hat{\sigma}_{\alpha}^{2}(f) \leq 
2 \sigma_{\alpha}^{2}(f)
+ (C^2+C) \ln\!\left(\frac{|\F(m)|2\log_{2}(2m)}{\delta}\right).
\end{equation*}

On the other hand, consider the set $\F'' \doteq \{ - f^2 : f \in \F \}$ 
and note that again we have $|\F''(m)| \leq |\F(m)|$.
Applying Lemma~\ref{lem:unlocalized-bernstein} with this $\F''$ and $\alpha'$, 
we have that, with probability at least $1-\delta/2$, 
every $f \in \F$ satisfies (for some numerical constant $C \geq 1$)
\begin{align}
-\hat{\sigma}_{\alpha}^{2}(f) 
& \leq -\sigma_{\alpha}^{2}(f)
+ C \sqrt{ \EE\!\left[\sum_{i=1}^{m} \alpha_i^4 f^4(W_i)\right] \ln\!\left(\frac{|\F(m)|2\log_{2}(2m)}{\delta}\right)} 
+ C \ln\!\left(\frac{|\F(m)|2\log_{2}(2m)}{\delta}\right)
\notag \\ & \leq - \sigma_{\alpha}^{2}(f)
+ C \sqrt{ \sigma_{\alpha}^{2}(f) \ln\!\left(\frac{|\F(m)|2\log_{2}(2m)}{\delta}\right)} 
+ C \ln\!\left(\frac{|\F(m)|2\log_{2}(2m)}{\delta}\right). \label{eqn:bvb-2}
\end{align}
If $\sigma_{\alpha}^{2}(f) \geq 4 C^2 \ln\!\left(\frac{|\F(m)|2\log_{2}(2m)}{\delta}\right)$ 
then the expression in \eqref{eqn:bvb-2} is at most 
\begin{equation*}
- \frac{1}{2} \sigma_{\alpha}^{2}(f)
+ C \ln\!\left(\frac{|\F(m)|2\log_{2}(2m)}{\delta}\right),
\end{equation*}
while if $\sigma_{\alpha}^{2}(f) \leq 4 C^2 \ln\!\left(\frac{|\F(m)|2\log_{2}(2m)}{\delta}\right)$ 
then the expression in \eqref{eqn:bvb-2} is at most 
\begin{equation*}
- \sigma_{\alpha}^{2}(f)
+ (2C^2+C) \ln\!\left(\frac{|\F(m)|2\log_{2}(2m)}{\delta}\right).
\end{equation*}
Thus, either way we have 
\begin{equation*}
-\hat{\sigma}_{\alpha}^{2}(f)
\leq 
- \frac{1}{2} \sigma_{\alpha}^{2}(f)
+ (2C^2+C) \ln\!\left(\frac{|\F(m)|2\log_{2}(2m)}{\delta}\right),
\end{equation*}
which implies 
\begin{equation*}
\hat{\sigma}_{\alpha}^{2}(f)
\geq 
\frac{1}{2} \sigma_{\alpha}^{2}(f)
- (2C^2+C) \ln\!\left(\frac{|\F(m)|2\log_{2}(2m)}{\delta}\right).
\end{equation*}

The lemma now follows by a union bound, so that these two events 
(each of probability at least $1-\delta/2$) 
occur simultaneously, with probability at least $1-\delta$.
\end{proof}

We will now show that Lemma~\ref{lem:uniform-bernstein} follows 
directly from Lemmas~\ref{lem:unlocalized-bernstein} and 
\ref{lem:bernstein-variance-concentration}.

\begin{proof}[Proof of Lemma~\ref{lem:uniform-bernstein}]
Set $\W = \X \times \Y$, $W_i = (X_i,Y_i)$, and $\alpha_i = 1$.
For each $h,h' \in \H$, define 
$f_{h,h'}(x,y) \doteq \ind{ h'(x) \neq y } - \ind{ h(x) \neq y }$.
Note that $\hat{\sigma}_{\alpha}^{2}(f_{h,h'}) = m\hat{\prob}_{S}(h \neq h')$.
$\F \doteq \{ f_{h,h'} : h,h' \in \H \}$ and note that 
$|\F(m)| \leq |\H(m)|^2$.
Applying Lemma~\ref{lem:unlocalized-bernstein} with this $\F$, 
we have that with probability at least $1-\delta/2$, 
every $h,h' \in \H$ satisfy (for some universal numerical constant $C \geq 1$)
\begin{equation}
\label{eqn:bernstein-proof-1}
\hat{\E}_{S}(h';h) \leq \EE\!\left[ \hat{\E}_{S}(h';h) \right]
+ C \sqrt{\EE\!\left[ \hat{\prob}_{S}(h \neq h') \right] \frac{1}{m}\ln\!\left(\frac{2|\H(m)|^2 \log_{2}(2m)}{\delta}\right)} 
+ \frac{C}{m} \ln\!\left(\frac{2|\H(m)|^2\log_{2}(2m)}{\delta}\right).
\end{equation}
Furthermore, Lemma~\ref{lem:bernstein-variance-concentration} implies 
that, with probability at least $1-\delta/2$, 
every $h,h' \in \H$ satisfy (for some universal numerical constant $C' \geq 1$)
\begin{equation}
\label{eqn:bernstein-proof-2}
\frac{1}{2} \EE\!\left[ \hat{\prob}_{S}(h \!\neq\! h') \right] 
- \frac{C'}{m} \ln\!\left(\frac{2|\H(m)|^{2}\log_{2}(2m)}{\delta}\right)
\leq \hat{\prob}_{S}(h \!\neq\! h')
\leq 2 \EE\!\left[ \hat{\prob}_{S}(h \!\neq\! h') \right] + \frac{C'}{m} \ln\!\left(\frac{2|\H(m)|^{2}\log_{2}(2m)}{\delta}\right).
\end{equation}
By a union bound, with probability at least $1-\delta$, 
every $h,h' \in \H$ satisfy 
both of \eqref{eqn:bernstein-proof-1} and \eqref{eqn:bernstein-proof-2}.
In particular, combining \eqref{eqn:bernstein-proof-1} with 
the left inequality in \eqref{eqn:bernstein-proof-2}, 
this also implies every $h,h' \in \H$ satisfy 
\begin{align}
\hat{\E}_{S}(h';h) 
& \leq \EE\!\left[ \hat{\E}_{S}(h';h) \right]
+ 2 C \sqrt{ \hat{\prob}_{S}(h \neq h') \frac{1}{m}\ln\!\left(\frac{2|\H(m)|^2 \log_{2}(2m)}{\delta}\right)}
\notag \\ & {\hskip 1cm} + \left( 2 \sqrt{C'} + 1 \right) C \frac{1}{m}\ln\!\left(\frac{2|\H(m)|^2 \log_{2}(2m)}{\delta}\right).
\label{eqn:bernstein-proof-3}
\end{align}

Finally, recall from Sauer's lemma \cite{VC:72,sauer1972density} 
that $|\H(m)| \leq \left(\frac{e \max\{m,\vc\}}{\vc}\right)^{\vc}$, 
and therefore 
\begin{align*}
\frac{1}{m} \ln\!\left(\frac{2|\H(m)|^2 \log_{2}(2m)}{\delta}\right) 
& \leq \frac{1}{m} \ln\!\left( \left( \frac{e \max\{m,\vc\}}{\vc} \right)^{2\vc} \frac{2 \log_{2}(2m)}{\delta}\right)
\\ & \leq \frac{1}{m} \ln\!\left( \left( \frac{e \max\{m,\vc\}}{\vc} \right)^{3\vc} \frac{4}{\delta}\right)
\leq C'' \Abound(m,\delta)
\end{align*}
for some numerical constant $C'' \geq 1$.

Altogether (and subtracting 
$\EE\!\left[ \hat{\E}_{S}(h';h) \right]$ 
and $\hat{\E}_{S}(h';h)$ from both sides of \eqref{eqn:bernstein-proof-1}
and \eqref{eqn:bernstein-proof-3}), 
we have that with probability at least $1-\delta$, every $h,h' \in \H$ 
satisfy 
\begin{equation*}
\EE\!\left[ \hat{\E}_{S}(h;h') \right]
\leq \hat{\E}_{S}(h;h')
+ 2 C \sqrt{C''} \sqrt{ \min\!\left\{ \EE\!\left[ \hat{\prob}_{S}(h \neq h') \right], \hat{\prob}_{S}(h \neq h') \right\} \Abound(m,\delta)} 
+ \left( 2 \sqrt{C'} + 1 \right) C C'' \Abound(m,\delta)
\end{equation*}
and 
\begin{equation*}
\frac{1}{2} \EE\!\left[ \hat{\prob}_{S}(h \neq h') \right] 
- C' C'' \Abound(m,\delta) 
\leq \hat{\prob}_{S}(h \neq h')
\leq 2 \EE\!\left[ \hat{\prob}_{S}(h \neq h') \right] + C' C'' \Abound(m,\delta),
\end{equation*}
which completes the proof.
\end{proof}

\section{Pooling is Optimal if Enough Tasks are Good}
\label{app:pooling-median}

While our results in Sections~\ref{sec:impossibilityoverview} and \ref{sec:nonadaptive} 
imply that, in general, one cannot achieve optimal rates by simply 
pooling all of the data and using the global ERM $\hat{h}_{Z}$, 
in this section we find that in some special cases this naive approach 
can actually be successful: namely, cases where \emph{most} of the tasks 
have $\rho_t$ below the cut-off value $\rho_{(t^*)}$ 
chosen by the optimization in the optimal proceedure 
from Section~\ref{sec:upper}.
We in fact show a general result for pooling, 
arguing that it always achieves a rate depending on the 
(weighted) \emph{median} value of $\brho{t}$, 
or more generally any \emph{quantile} of $\brho{t}$ values.

\begin{theorem}
[Pooling Beyond $\beta=1$]
\label{thm:pooling-median}
For any $\alpha \in (0,1]$, let $t(\alpha)$ 
be the smallest value in $[N+1]$ 
such that $\sum_{t \in [t(\alpha)]} n_{(t)} \geq \alpha \sum_{t=1}^{N+1} n_t$.
Then, for any $\delta \in (0,1)$, with probability at least $1-\delta$ we have 
\begin{equation*}
\E_{\Qt}\!\left( \hat{h}_{Z} \right) 
\leq C_{\rho} \left( C \frac{\vc \log\!\left(\frac{1}{\vc} \sum_{t=1}^{N+1} n_{t} \right) + \log(1/\delta) }{\sum_{t=1}^{N+1} n_{t}} \right)^{1/(2-\beta)\brho{t(\alpha)}}
\end{equation*}
for a constant $C = (32 \Aconst^2 / \alpha)^{2-\beta} C_{\beta}$.
\end{theorem}
\begin{proof}
Letting $\NNI{[N+1]} \doteq \sum_{t=1}^{N+1} n_t$,
$\bar{P} \doteq \NNI{[N+1]}^{-1} \sum_{t=1}^{N+1} n_t P_t$, 
and 
$\NNI{\alpha} \doteq 
\sum_{t \in [t(\alpha)]} n_{(t)}$,
we have 
\begin{align*}
\E_{\bar{P}}(\hat{h}_{Z}) 
& \geq \NNI{[N+1]}^{-1} \sum_{t=1}^{N+1} n_{t} \left( C_{\rho}^{-1} \E_{\Qt}(\hat{h}_{Z}) \right)^{\rho_{t}}
\\ & \geq \frac{\NNI{\alpha}}{\NNI{[N+1]}} \sum_{t \in [t(\alpha)]} \frac{n_{(t)}}{\NNI{\alpha}} \left( C_{\rho}^{-1} \E_{\Qt}(\hat{h}_{Z}) \right)^{\rho_{(t)}}
\geq \alpha \left( C_{\rho}^{-1} \E_{\Qt}(\hat{h}_{Z}) \right)^{\brho{t(\alpha)}},
\end{align*}
where the third inequality 
is due to Jensen's inequality.
This implies 
\begin{equation}
\label{eqn:pooling-median-Q-bound}
\E_{\Qt}(\hat{h}_{Z}) \leq C_{\rho} \left( (1/\alpha) \E_{\bar{P}}(\hat{h}_{Z}) \right)^{1/\brho{t(\alpha)}}.
\end{equation}
Lemma~\ref{lem:Ealpha-bound} implies 
that with probability at least $1-\delta$, 
\begin{equation*}
\E_{\bar{P}}(\hat{h}_{Z}) 
\leq 32 \Aconst^2 (C_{\beta} \Abound(\NNI{[N+1]},\delta))^{1/(2-\beta)}.
\end{equation*}
Combining this with \eqref{eqn:pooling-median-Q-bound} completes the proof.
\end{proof}

For instance, if all $n_{t}$ 
are equal some common value $n$, 
and we take 
$\alpha = 1/2$, 
then we can take 
$t(\alpha) = \lceil (N+1)/2 \rceil$, 
so that the optimal rate will 
be achieved as long as at least 
half of the tasks have $\brho{t}$ 
below the value $\brho{t^*}$
for $t^*$ the minimizer of the 
bound in Theorem~\ref{thm:minimax}.

Optimizing the bound in 
Theorem~\ref{thm:pooling-median} 
over the choice of $\alpha$
yields the following result.

\begin{corollary}[General Pooling Bound]
\label{cor:general-pooling}
For any $\delta \in (0,1)$, 
with probability at least $1-\delta$ 
we have 
\begin{equation*}
\E_{\Qt}\!\left( \hat{h}_{Z} \right) 
\leq \min\limits_{t \in [N+1]}
C_{\rho} \left( C \frac{\vc \log\!\left(\frac{1}{\vc} \sum_{s=1}^{N+1} n_{s} \right) + \log(1/\delta) }{ \left( \sum_{s=1}^{t} n_{(s)} \right)^{2-\beta}\left(\sum_{s=1}^{N+1} n_{s}\right)^{-(1-\beta)}} \right)^{1/(2-\beta)\brho{t}}
\end{equation*}
for a constant $C = (32 \Aconst^2)^{2-\beta} C_{\beta}$.
\end{corollary}

\begin{remark}
In particular, note that 
this result recovers the 
bound of Theorem~\ref{thm:beta-1-union}
in the special case of $\beta=1$.
\end{remark}

\section{Different Optimal Aggregations of Tasks in Multitask}\label{app:asymmetry}
We present a proof of Theorem \ref{theo:asymmetry} in this section.
Recall that this result states that different choices of target in the same multitask setting can induce different optimal aggregation of tasks. As a consequence, the naive approach of pooling of all tasks 
can adversely affect target risk even when all $\hstar$'s are the same.  

We employ a similar construction to that of Section \ref{sec:nonadaptive}. 

\paragraph{Setup.} We again build \sk{on distributions} supported on 2 datapoints $x_0, x_1$.  W.l.o.g., assume that $x_0$ has label 1. Let $n_p, n_q \geq 1$, ${0\leq \beta < 1}$, and define  
$\epsilon \doteq \np^{-1/(2-\beta)}$. Let $\sigma\in \braces{\pm 1}$ -- which we will often abbreviate as $\pm$. In all that follows, we let $\eta_{\mu}(X)$ denote the regression function $\prob_{\mu}[Y = 1\mid  X]$ under distribution $\mu$.
\sk{
\begin{itemize} 
\item {\bf Target $\Qt_\sigma = \Qt_X \times \Qt^\sigma_{Y|X}$:} 
Let $\Qt_X(x_1) = {1}/{2}$, $\Qt_X(x_0) = {1}/{2}$; finally $\Qt^{\sigma}_{Y|X}$ is determined by \\
$\eta_{\Qt, \sigma}(x_1) = 1/2 + \sigma\cdot ({1}/{4})$, and $\eta_{\Qt, \sigma}(x_0) = 1$.
\item {\bf Source $P_\sigma = P_X \times P^\sigma_{Y | X}$:} Let $P_X(x_1) = \epsilon^\beta $, $P_X(x_0) = 1- \epsilon^\beta$; finally $P^\sigma_{Y | X}$ is determined by \\
$\eta_{P, \sigma}(x_1) = 1/2 + \sigma\cdot {c_2} \epsilon^{1-\beta}$, and $\eta_{P, \sigma}(x_0) = 1$, 
{for an appropriate constant $c_1$ specified in the proof.}
\end{itemize}
}

\begin{proof}[{Proof of Theorem \ref{theo:asymmetry}}] 
Let $P \doteq P_{-}$ and $\Qt \doteq \Qt_{-}$. The above construction then ensures that any $h$ s.t. $h(x_1) = +1$ has excess error $\E_{\Qt}(h) = 1/4$. Thus we just have to show that number of $+1$ labels at $x_1$ exceeds the number of $-1$ labels at $x_1$ with non-zero probability (so that both $\hat h_Z, \hat h_{Z_P}$ would pick $+1$ at $x_1$). Let ${\hat n}_{P, 1}$ denote the number of samples from $P$ at $x_1$, and ${\hat n}_{P, 1}^+$ denote those samples from $P$ having label $+1$. Notice that if ${\hat n}_{P, 1}^+ > \frac{1}{2}\paren{{\hat n}_{P, 1} + n_\Qt}$, then $+1$ necessarily dominates $-1$ at $x_1$.

Now, conditioned on ${\hat n}_{P, 1}$, 
${\hat n}_{P, 1}^+$ is distributed as $\text{Binomial}({\hat n}_{P, 1}, p)$ with $p = \eta_{P,-}(x_1)$. Applying Lemma \ref{lem:Slud's}, we have 
\begin{align*} 
\prob\paren{{\hat n}_{P, 1}^+ > n_{P,1} \cdot p + \sqrt{{\hat n}_{P, 1}\cdot p(1-p)}} \geq \frac{1}{12}. 
\end{align*}
In other words, under the above binomial event, $+1$ dominates whenever (the second inequality below holds)
\begin{align*} 
 \sqrt{{\hat n}_{P, 1}\cdot p(1-p)} \geq \frac{1}{4}\sqrt{{\hat n}_{P, 1}} \geq {\hat n}_{P, 1}\paren{\frac{1}{2} - p} + \frac{1}{2}n_\Qt,
\end{align*}
in other words, if we have both ${{\hat n}_{P, 1}}\cdot c_1^2 \cdot \epsilon^{2(1-\beta)}\leq \frac{1}{64}$ and $n_\Qt^2 \leq \frac{1}{4}{{\hat n}_{P, 1}}$. Let $\bE_P\doteq \braces{ \frac{1}{2}\expec \brackets{{\hat n}_{P, 1}} \leq {\hat n}_{P, 1} \leq 2 \expec \brackets{{\hat n}_{P, 1}}  }$, where $\expec \brackets{{\hat n}_{P, 1}} = n_P\cdot P_X(x_1) = n_P\cdot \epsilon^\beta$. Under this event, we just need that $2c_1^2 \cdot n_P \cdot \epsilon^{(2-\beta)} = 2c_1^2 \leq \frac{1}{64}$, and 
that {$n_\Qt^2 \leq \frac{1}{8}n_P\cdot \epsilon^\beta = \frac{1}{8} n_P^{(2- 2\beta)/(2-\beta)}$, requiring $\beta < 1$}. Hence, integrating over ${\hat n}_{P, 1}$, we have that 
\begin{align*} 
\prob \paren{{\hat n}_{P, 1}^+ > \frac{1}{2}\paren{{\hat n}_{P, 1} + n_\Qt}} \geq \frac{1}{12} \prob\paren{\bE_P} > 0,
\end{align*}
where $\prob\paren{\bE_P}$ is bounded below by a multiplicative Chernoff whenever 
{$\expec \brackets{{\hat n}_{P, 1}}= n_P^{(2- 2\beta)/(2-\beta)} \geq 1$}. 
\end{proof}

\section{Supporting Results for Section \ref{sec:nonadaptive} on Impossibility of Adaptivity} \label{app:nonadapt} 

Recall that from our construction, we set 
$$\Gamma_\sigma = \paren{\alpha_P P_\sigma^n + \alpha_Q \Q_\sigma^n}^N \times \Qt_\sigma^{n_\Qt}.$$

\begin{proof}[Proof of Proposition \ref{prop:likelihoodRatio}]
Let $\Gamma_{\sigma, \alpha} \doteq {\alpha_P P_\sigma^n + \alpha_Q \Q_\sigma^n}$. Clearly 
    $$\frac{\Gamma_{+}(Z)}{\Gamma_{-}(Z)} = \frac{\Qt_{+}(Z_{N+1})}{\Qt_{-}(Z_{N+1})}\cdot \prod_{t=1}^N\frac{\Gamma_{+, \alpha}(Z_t)}{\Gamma_{-, \alpha}(Z_t)}.$$

Define ${\hat n}_{t, 0}$, ${\hat n}_{t, +}$, and ${\hat n}_{t, -}$ as the number of points ($X_{t, i}, i \in [n]$) in $Z_t$ which, respectively fall on $x_0$, or fall on $x_1$ with $Y_{t, i} = +1$, or fall on $x_1$ with $Y_{t, i} = -1$. 
Recall that ${\hat n}_\sigma = \sum_{t\in [N]} {\hat n}_{t, \sigma}$ for $\sigma = \pm$.
Next, define ${\cal Z}_\sigma \doteq \braces{Z_t, t \in [N]: {\hat n}_{t, \sigma} = n}$, i.e., the set of $\sigma$-homogeneous vectors. 
We have: 
\begin{align*} 
\Gamma_{+, \alpha}(Z_t) &= \alpha_P\cdot \paren{\eta_{P, -}^{{\hat n}_{t, -}}(x_1)\cdot \eta_{P, +}^{{\hat n}_{t, +}}(x_1)\cdot P_X^{\paren{{\hat n}_{t, +} + {\hat n}_{t, -}}}(x_1)\cdot P_X^{{\hat n}_{t, 0}}(x_0)} + \alpha_\Q\cdot \ind{Z_t \in {\cal Z}_{+}}, \\
\Gamma_{-, \alpha}(Z_t) &= \alpha_P\cdot \paren{\eta_{P, -}^{{\hat n}_{t, +}}(x_1)\cdot \eta_{P, +}^{{\hat n}_{t, -}}(x_1)\cdot P_X^{\paren{{\hat n}_{t, +} + {\hat n}_{t, -}}}(x_1)\cdot P_X^{{\hat n}_{t, 0}}(x_0)} + \alpha_\Q\cdot \ind{Z_t \in {\cal Z}_{-}}. 
\end{align*} 
We can then consider homogeneous and non-homogeneous vectors separately. Let ${\cal Z}_\pm \doteq {\cal Z}_{+}\cup {\cal Z}_{-}$,
\begin{align} 
\prod_{Z_t \in {\cal Z}_{\pm}}\frac{\Gamma_{+, \alpha}(Z_t)}{\Gamma_{-, \alpha}(Z_t)} 
&= \frac{\alpha_P^{{\hat N}_{-}}\paren{\eta_{P, -}(x_1)\cdot P_X(x_1)}^{n\cdot {\hat N}_{-}}}{\paren{\alpha_P\paren{\eta_{P, +}(x_1)\cdot P_X(x_1)}^n + \alpha_\Q}^{{\hat N}_{-}}}\cdot 
\frac{\paren{\alpha_P\paren{\eta_{P, +}(x_1)\cdot P_X(x_1)}^n + \alpha_\Q}^{{\hat N}_{+}}}{\alpha_P^{{\hat N}_{+}}\paren{\eta_{P, -}(x_1)\cdot P_X(x_1)}^{n\cdot {\hat N}_{+}}} \nonumber \\
&= \alpha_P^{\paren{{\hat N}_{-} - {\hat N}_{+}}}\cdot\paren{\eta_{P, -}(x_1)\cdot P_X(x_1)}^{n\paren{{\hat N}_{-} - {\hat N}_{+}}}\cdot
\paren{\alpha_P\paren{\eta_{P, +}(x_1)\cdot P_X(x_1)}^n + \alpha_\Q}^{{\hat N}_{+} - {\hat N}_{-}}
\label{eq:homogeneous}
\end{align} 
where we broke up the product over $\cal{Z}_{-}$ and $\cal{Z}_{+}$. Now, canceling out similar terms in the fraction, 
\begin{align}
\prod_{Z_t \notin {\cal Z}_{\pm}}\frac{\Gamma_{+, \alpha}(Z_t)}{\Gamma_{-, \alpha}(Z_t)} 
&= \paren{\frac{\eta_{P, -}(x_1)}{\eta_{P, +}(x_1)}}^{\paren{{\hat n}_{-} - n\cdot {\hat N}_{-}}}\cdot 
\paren{\frac{\eta_{P, +}(x_1)}{\eta_{P, -}(x_1)}}^{\paren{{\hat n}_{+} - n\cdot {\hat N}_{+}}} \nonumber \\
&=\paren{\frac{\eta_{P, +}(x_1)}{\eta_{P, -}(x_1)}}^{\paren{{\hat n}_{+} - {\hat n}_{-}}}\cdot
\paren{\frac{\eta_{P, -}(x_1)}{\eta_{P, +}(x_1)}}^{n \paren{{\hat N}_{+} - {\hat N}_{-} }} \label{eq:non-homogeneous}
\end{align}
Now, the second factor in \eqref{eq:non-homogeneous} can be expanded as follows to cancel out the second factor in \eqref{eq:homogeneous}:  
\begin{align*} 
\paren{\frac{\eta_{P, -}(x_1)}{\eta_{P, +}(x_1)}}^{n \paren{{\hat N}_{+} - {\hat N}_{-} }}
= \paren{\frac{\eta_{P, -}(x_1)\cdot P_X(x_1)}{\eta_{P, +}(x_1)\cdot P_X(x_1)}}^{n\paren{{\hat N}_{+} - {\hat N}_{-} }}. 
\end{align*}
That is, we have: 
\begin{align*} 
\paren{\frac{\eta_{P, -}(x_1)}{\eta_{P, +}(x_1)}}^{n \paren{{\hat N}_{+} - {\hat N}_{-} }}\cdot 
\prod_{Z_t \in {\cal Z}_{\pm}}\frac{\Gamma_{+, \alpha}(Z_t)}{\Gamma_{-, \alpha}(Z_t)} 
= \paren{\frac{\alpha_P\paren{\eta_{P, +}(x_1)\cdot P_X(x_1)}^n + \alpha_\Q}{\alpha_P\paren{\eta_{P, +}(x_1)\cdot P_X(x_1)}^n}}^{\paren{{\hat N}_{+} - {\hat N}_{-}}}. 
\end{align*} 
In other words, we have: 
\begin{align*}
\prod_{t=1}^N\frac{\Gamma_{+, \alpha}(Z_t)}{\Gamma_{-, \alpha}(Z_t)}  
= \paren{\frac{\eta_{P, +}(x_1)}{\eta_{P, -}(x_1)}}^{\paren{{\hat n}_{+} - {\hat n}_{-}}}\cdot \paren{\frac{\alpha_P\paren{\eta_{P, +}(x_1)\cdot P_X(x_1)}^n + \alpha_\Q}{\alpha_P\paren{\eta_{P, +}(x_1)\cdot P_X(x_1)}^n}}^{\paren{{\hat N}_{+} - {\hat N}_{-}}}.
\end{align*}
We then conclude by noticing that each of the above fractions is greater than $1$.  
\end{proof}

\begin{proof}[Proof of Proposition \ref{prop:targetlikelihood}] 
For $t = N+1$, define ${\hat n}_{t, 0}$, ${\hat n}_{t, +}$, and ${\hat n}_{t, -}$ as the number of points ($X_{t, i}, i \in [n]$) in $Z_t$ which, respectively fall on $x_0$, or fall on $x_1$ with $Y_{t, i} = +1$, or fall on $x_1$ with $Y_{t, i} = -1$. Thus, for $t = N+1$ fixed, we have 
\begin{align} 
\frac{\Qt_{+}(Z_{N+1})}{\Qt_{-}(Z_{N+1})} = 
\frac{\Qt_X^{{\hat n}_{t, 0}}(x_0)\cdot\Qt_X^{({\hat n}_{t, +} + {\hat n}_{t, -})}(x_1)\cdot \eta_{\Qt, +}^{{\hat n}_{t, +}}(x_1) \cdot \eta_{\Qt, -}^{{\hat n}_{t, -}}(x_1) }
{\Qt_X^{{\hat n}_{t, 0}}(x_0)\cdot\Qt_X^{({\hat n}_{t, +} + {\hat n}_{t, -})}(x_1)\cdot \eta_{\Qt, +}^{{\hat n}_{t, -}}(x_1) \cdot \eta_{\Qt, -}^{{\hat n}_{t, +}}(x_1)}
= \paren{\frac{\eta_{\Qt, +}(x_1)}{\eta_{\Qt, -}(x_1)}}^{({\hat n}_{t, +} - {\hat n}_{t, -})}.\label{eq:targetratio}
\end{align}
Now \eqref{eq:targetratio} $\geq 1$ whenever ${\hat n}_{t, +} \geq {\hat n}_{t, -}$, so we proceed to bounding the probability of this event under $\Qt_{-}$. In particular, when $n = 1$, the event has probability at least 
$\prob({\hat n}_{t, 0} = 1) = \Qt_X(x_0) = 1- \frac{1}{2}= \frac{1}{2}$. Assume henceforth that $n > 1$, and let ${\hat n}_{t} \doteq {\hat n}_{t, +} + {\hat n}_{t, -}$. Conditioned on ${\hat n}_{t}$, ${\hat n}_{t, +}$ is distributed as $\text{Binomial}({\hat n}_{t}, p)$, with 
$p = \eta_{\Qt, -}(x_1)$. By the anticoncentration Lemma \ref{lem:Slud's}, we then have 

\begin{align} 
\prob\paren{{\hat n}_{t, +} > {\hat n}_{t} \cdot p + \sqrt{{\hat n}_{t}\cdot p(1-p)}} \geq \frac{1}{12}, \label{eq:anticoncTarget}
\end{align}

so the event ${\hat n}_{t, +} \geq {\hat n}_{t, -}$ holds whenever 
\begin{align} 
\sqrt{{\hat n}_{t}\cdot p(1-p)} \geq \frac{1}{2} {\hat n}_{t} - {\hat n}_{t}\cdot p &= {\hat n}_{t}\paren{\frac{1}{2} - p}
= {\hat n}_{t} \cdot c_0\cdot \epsilon_0^{1-\beta}, \text{in other words, whenever}\nonumber \\
{\hat n}_{t} \cdot c_0^2\cdot \epsilon_0^{2(1-\beta)} &\leq \frac{1}{8}, \text{ (assuming }{c_0 \leq 1/4}). \label{eq:anticonc-Target-Cond}
\end{align}

Now, consider the event 
$\bE_{\Qt}\doteq \braces{{\hat n}_{t} \leq 2\expec [{\hat n}_{t}]}$, where $\expec [{\hat n}_{t}] = n_\Qt\cdot \Qt_{X}(x_1) = \frac{1}{2} n_\Qt\cdot \epsilon_0^{\beta}$.
Under $\bE_{\Qt}$, \eqref{eq:anticonc-Target-Cond}
is satisfied whenever 
{$c_0^2   \leq \frac{1}{8}$}, recalling $\epsilon_0 \doteq n_\Qt^{-1/(2-\beta)}$. In other words, under this condition, we have 
\begin{align*} 
\prob\paren{{\hat n}_{t, +} \geq {\hat n}_{t, -}} \geq \expec\brackets{\prob\paren{{\hat n}_{t, +} \geq {\hat n}_{t, -} \mid {\hat n}_{t}}\ind{\bE_\Qt}}
\geq \frac{1}{12}\prob(\bE_\Qt). 
\end{align*}

Finally, by multiplicative Chernoff, the event $\bE_{\Qt}$ holds with probability at least {$1- \exp(-1/6) > 1/7$}. 
\end{proof}

\section{Auxiliary Lemmas}
\label{app:aux-lemmas}

The following propositions are taken verbatim from \citep*{hanneke2019value}.

\begin{proposition} [Thm 2.5 of \cite{tsybakov2009introduction}] \label{prop:tsy25} Let $\{ \Pi_{h} \}_{h \in \Hyp}$ be a family of distributions indexed over a subset $\Hyp$ of a semi-metric $( \mathcal{F}, \semiMetric )$. Suppose $\exists \, h_0, \ldots, h_{M} \in \Hyp$, where $M \geq 2$, such that:
\begin{flalign*} 
\qquad {\rm (i)} \quad  &\semiDist{h_{i}}{h_{j}} \geq 2 s > 0, \quad \forall 0 \leq i < j \leq M,  & \\
\qquad {\rm (ii)} \quad  & \Pi_{h_i} \ll \Pi_{h_0} \quad \forall i \in  [M], \text{ and the average  KL-divergence to } \Pi_{h_0} \text{ satisfies } & \\
& \qquad 
\frac{1}{M} \sum_{i = 1}^{M} \KLDiv{\Pi_{h_i}}{ \Pi_{h_0}} \leq \alpha \log M, \text{ where } 0 < \alpha < 1/8.
\end{flalign*}
Let $Z\sim\Pi_{h}$, and let $\hat h : Z \mapsto \mathcal{F}$ denote any \emph{improper} learner of $h\in \Hyp$. We have for any $\hat h$: 
\begin{equation*}
\sup_{h \in \Hyp} \Pi_{h} \left( \semiDist{\hat h(Z)}{h} \geq s \right) \geq \frac{\sqrt{M}}{1 + \sqrt{M}} \left( 1 - 2 \alpha - \sqrt{\frac{2 \alpha}{\log(M)}} \right) \geq \frac{3 - 2 \sqrt{2}}{8}.
\end{equation*}
\end{proposition}

\begin{proposition} [Varshamov-Gilbert bound] \label{lem:VGBound}
Let $d \geq 8$. Then there exists a subset $\{ \sigma_0, \ldots, \sigma_{M}\}$ of $\{-1 ,1 \}^{d}$ such that $\sigma_0 = (1,\ldots,1)$,
\begin{equation*}
{\rm dist}(\sigma_{i},\sigma_{j}) \geq \frac{d}{8}, \quad \forall\,  0 \leq i < j \leq M, \quad \text{and} \quad M \geq 2^{d / 8},
\end{equation*}
where ${\rm dist}(\sigma,\sigma') \doteq {\rm card}(\{ i \in [m] :  \sigma(i) \neq \sigma'(i) \})$ is the Hamming distance.
\end{proposition}

\begin{lemma} [A basic KL upper-bound]
\label{lem:klbound} 
For any $0<p, q<1$, we let $\KLDiv{p}{q}$ 
denote $\KLDiv{{\rm Ber}(p)}{{\rm Ber}(q)}$. 
Now let $0<\epsilon<1/2$ and let $z\in \{ -1, 1\}$. We have 
$$\KLDiv{1/2 + (z/2)\cdot \epsilon\, }{\, 1/2 - (z/2)\cdot \epsilon} 
\leq c_0\cdot \epsilon^2, \text{ for some } c_0 \text{ independent of } \epsilon.$$
\end{lemma}

\end{document}